\documentclass[10pt, twocolumn, letterpaper]{article}
\usepackage[nonatbib, preprint]{cig_twocolumns}

\PassOptionsToPackage{numbers}{natbib}

\usepackage[utf8]{inputenc}
\usepackage[T1]{fontenc}

\usepackage{amsmath, amssymb, amsfonts, mathtools, bm}
\usepackage{nicefrac}
\usepackage{bbm}
\usepackage{siunitx}
\usepackage{pifont}
\usepackage{algorithm}
\usepackage{algpseudocode}
\usepackage{amsthm}
\usepackage{calc}
\usepackage{cite}
\usepackage[table]{xcolor}
\definecolor{darkblue}{rgb}{0, 0, 0.85}
\definecolor{lightgreen}{rgb}{.85,1,.85}
\definecolor{lightred}{rgb}{1,.85,.85}
\definecolor{lightblue}{rgb}{.85,.85,1}
\definecolor{pink}{HTML}{EB346F}
\definecolor{cvprblue}{rgb}{0.21,0.49,0.74}

\usepackage{hyperref}
\hypersetup{
    colorlinks = true,
    citecolor = cvprblue,
    linkcolor = black
}

\usepackage{graphicx}
\usepackage{tabularx, threeparttable, tabularray}
\UseTblrLibrary{booktabs}
\usepackage{booktabs}
\usepackage{multirow}
\usepackage{wrapfig}
\usepackage{arydshln}
\usepackage{hhline}
\usepackage{caption}
\usepackage{diagbox}

\usepackage{microtype}
\usepackage{setspace}
\usepackage{lipsum}
\usepackage{soul}

\usepackage{tikz}
\usepackage{enumitem}
\theoremstyle{plain}

\newtheorem{definition}{Definition}
\newtheorem{theorem}{Theorem}
\newtheorem{lemma}{Lemma}

\newtheorem{assumption}{Assumption}

\newcommand{\hlgreen}[1]{{\sethlcolor{lightgreen}\hl{#1}}}

\newcommand{\hlbest}[1]{\colorbox{lightgreen}{\makebox(22,6){\textbf{#1}}}}

\DeclarePairedDelimiterX{\infdivx}[2]{(}{)}{#1\;\delimsize\|\;#2}

\renewcommand{\vec}[1]{\bm{#1}}

\newcommand{\X}{{\vec{X}}}  
\newcommand{\Y}{{\vec{Y}}}  
\newcommand{\Z}{{\vec{Z}}}  
\newcommand{\Q}{{\vec{Q}}}

\def\x{\vec{x}}  
\def\y{\vec{y}}  
\def\z{\vec{z}}  
\def\e{\vec{e}}  
\def\s{\vec{s}}

\def\P{\mathbb{P}}
\def\Q{\mathbb{Q}}

\newcommand{\normpdf}{{\mathcal{N}}}

\def\argmin{\mathop{\mathsf{arg\,min}}}

\usepackage{pifont}

\newcolumntype{x}[1]{>{\centering\arraybackslash}p{#1pt}}
\newcolumntype{y}[1]{>{\raggedright\arraybackslash}p{#1pt}}
\newcolumntype{z}[1]{>{\raggedleft\arraybackslash}p{#1pt}}

\definecolor{selectcolor}{gray}{.85}

\renewcommand{\paragraph}[1]{\vspace{1.25mm}\noindent\textbf{#1}}

\newcommand{\app}{\raise.17ex\hbox{$\scriptstyle\sim$}}

\title{Stochastic Generative Plug-and-Play Priors}

\vspace{5em}
\author{%
\normalsize Chicago~Y.~Park$^1$ \quad Edward~P.~Chandler$^1$ \quad Yuyang~Hu$^2$ \quad Michael~T.~McCann$^3$ \\[0.5em]  Cristina~Garcia-Cardona$^3$ \quad  Brendt~Wohlberg$^3$ \quad  Ulugbek~S.~Kamilov$^1$\\[0.7em]
\small $^1$\textnormal{UW--Madison} \quad $^2$\textnormal{WashU} \quad $^3$\textnormal{Los Alamos National Laboratory} \\ [0.7em]
\footnotesize \texttt{kamilov@wisc.edu}
}

\usepackage{geometry}
\usepackage{xcolor}
\usepackage{listings}
\usepackage{algorithm}
\usepackage{caption}
\usepackage[table]{xcolor}

\definecolor{codePink}{HTML}{D63384}  
\definecolor{codeTeal}{HTML}{458588}  
\definecolor{codeBlue}{HTML}{2A52BE}  
\definecolor{codeBlack}{HTML}{000000} 

\def\prox{\operatorname{prox}}  
\def\x{\vec{x}}  
\def\y{\vec{y}}  
\def\z{\vec{z}}  
\def\w{\vec{w}}  
\def\e{\vec{e}}  
\def\v{\vec{v}}  
\def\Dsf{\mathsf{D}}
\def\saddlept{\bar{\vec{x}}}
\DeclareCaptionFormat{algor}{#1#2#3}
\usepackage{booktabs}
\usepackage[table]{xcolor}
\usepackage{array}
\usepackage{graphicx}
\usepackage{array}

\usepackage{booktabs}
\usepackage{xcolor}
\usepackage{pifont}
\usepackage{array}

\begin{document}

\maketitle

\begin{abstract}
Plug-and-play (PnP) methods are widely used for solving imaging inverse problems by incorporating a denoiser into an optimization algorithm.
Score-based diffusion models (SBDMs) have recently demonstrated strong generative performance via a denoiser trained across a wide range of noise levels.
Despite their shared reliance on denoisers, it remains unclear how to systematically use SBDMs as priors within the PnP framework without relying on reverse diffusion sampling.
In this paper, we establish a score-based interpretation of PnP that justifies using pretrained SBDMs directly within PnP algorithms.
Building on this connection, we introduce a stochastic generative PnP (SGPnP) framework that uses stochastic re-noising to better leverage the broad noise-level conditioning of pretrained SBDMs, significantly improving reconstruction quality in severely ill-posed inverse problems.
We show that noise injection induces optimization on a Gaussian-smoothed objective and enables avoidance of strict saddle points under suitable conditions.
Experiments on challenging inverse tasks, such as multi-coil MRI reconstruction and large-mask natural image inpainting, demonstrate consistent improvement over conventional PnP methods and achieve performance competitive with diffusion-based solvers.
Code is available at \href{https://github.com/uw-cig/SGPnP}{\textcolor{cvprblue}{\text{https://github.com/uw-cig/SGPnP}}}.
\end{abstract}

\section{Introduction}

The recovery of an unknown image from incomplete and noisy measurements is fundamental to computational imaging. Such inverse problems arise in a wide range of applications, including image deblurring, super-resolution, inpainting, and magnetic resonance imaging (MRI).

Plug-and-play (PnP) priors~\cite{Venkatakrishnan.etal2013} is a framework for solving imaging inverse problems by alternating between enforcing measurement consistency and imposing prior information through a learned denoiser.
By replacing an explicit analytical prior (e.g., total variation) with a pre-trained denoiser, PnP enables the use of powerful learned image statistics while retaining the flexibility of optimization-based solvers. This modularity has made PnP a popular approach across a broad range of imaging tasks.

Despite their success, PnP methods face challenges in severely ill-posed inverse problems (see Figure~\ref{fig:box_comparison}).
Two factors contribute to this limitation.
First, a distribution mismatch arises because intermediate PnP iterates contain structured artifacts rather than additive Gaussian noise.
As a result, the denoiser is often used outside the noise regime for which it was trained~\cite{eslahi2022improved}.
Second, the denoisers typically used in PnP are limited in expressiveness.
Classical denoisers such as block-matching and 3D filtering (BM3D)~\cite{dabov2007bm3d} handle large noise levels but rely on hand-crafted image models rather than learned image statistics, while learned denoisers are trained over a limited noise range.
Neither captures the image distribution well enough to handle severely degraded images, where substantial ambiguity or missing information necessitates strong priors.
A recent stochastic re-noising strategy~\cite{renaud2024snore} partially addresses the distribution mismatch by injecting noise before denoising.
However, this approach still relies on low-noise denoisers and yields limited improvements in strongly ill-posed settings, as illustrated in Figure~\ref{fig:box_comparison}.

\begin{figure}[t]
\begin{center}
\includegraphics[width=0.49\textwidth]{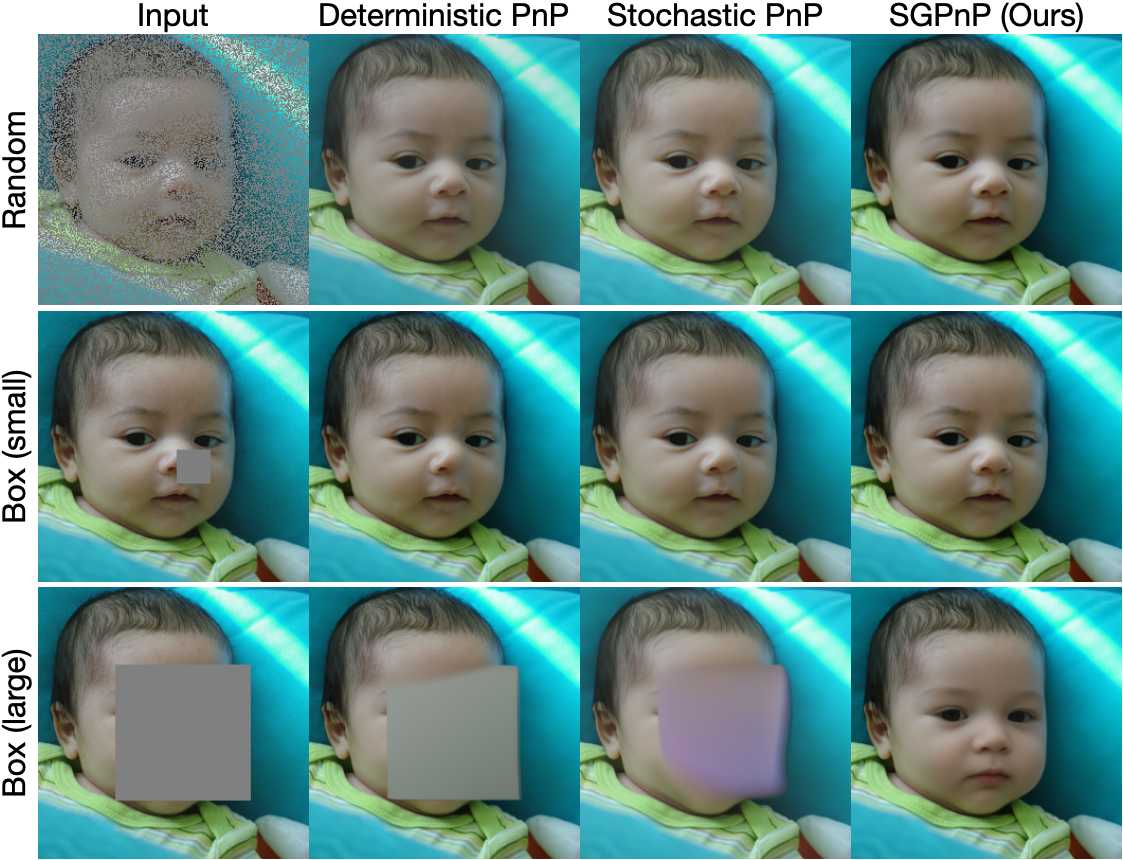}
\end{center}
\caption{
Comparison of deterministic PnP~\cite{zhang2021dpir}, stochastic PnP~\cite{renaud2024snore}, and SGPnP on three inpainting tasks: random masks, box masks with small missing regions, and box masks with large missing regions. Deterministic PnP succeeds only on random inpainting, while stochastic PnP extends this to small box masks but fails for large missing regions. In contrast, the proposed SGPnP approach performs reliably across all three settings.
}
\vspace{-.3cm}
\label{fig:box_comparison}
\end{figure}

Score-based diffusion models (SBDMs)~\cite{ho_NEURIPS2020_ddpm, song2021sde, park2024randomwalks} have recently emerged as a powerful framework for image generation.
Unlike conventional denoisers limited to low-noise regimes~\cite{Venkatakrishnan.etal2013, zhang2021dpir, romano2017RED}, SBDMs are trained across a broad range of noise levels.
By learning the score---the gradient of the log-density of noisy data---SBDMs enable iterative sampling that transforms noise into realistic images.
Motivated by their generative capability, recent work~\cite{laumont2022pnpula, bouman2023generativePnP, zhu2023DiffPIR, kawar2022ddrm, chung2023dps, mike2023score, sun2024provablePMC, wu2024principledPnP, xu2024provably_dps_pnp, coeurdoux2024pnpsplitgibbssampler, faye2024REDbayesian, gulle2025consistency} has explored diffusion-based solvers for inverse problems by incorporating measurement consistency into the sampling process, achieving impressive results in severely ill-posed scenarios such as box inpainting (see \cite{daras2024survey} for a review).

Despite the shared reliance on denoisers of both PnP and SBDMs, it remains unclear how to systematically incorporate pretrained SBDMs as priors within PnP algorithms without relying on reverse diffusion sampling at every optimization step.
This paper addresses this gap by providing a complete and principled framework for leveraging pre-trained SBDMs as effective priors within PnP iterations.
Our contributions are as follows:

\begin{itemize}[leftmargin=*]
\item \textbf{Stochastic generative PnP framework:} We introduce
\emph{stochastic generative PnP (SGPnP)}, which re-noises the iterate at the denoiser input
so that pretrained SBDMs can be queried across their full noise range and integrated into
optimization-based PnP without reverse diffusion process.

\item \textbf{Theoretical guarantees for SGPnP:}
We establish the first strict saddle-point avoidance guarantee in the PnP literature, showing that, under explicitly stated conditions, noise injection enables SGPnP to avoid strict saddle points. Furthermore, under an annealed noise schedule, SGPnP converges to a stationary point of the exact (un-smoothed) MAP objective in expectation.

\item \textbf{Empirical improvement via SGPnP:} With the same score-based prior
as its deterministic counterpart, SGPnP significantly improves reconstruction quality on box
inpainting, super-resolution, deblurring, and CS-MRI.

\end{itemize}

This paper extends our conference paper~\cite{park2024scorepnp}, which established how to use pre-trained SBDM denoisers within deterministic PnP methods, with three new contributions. First, we introduce the SGPnP framework, which uses stochastic re-noising to significantly improve reconstruction quality in severely ill-posed inverse problems.
Second, we provide theoretical guarantees for SGPnP, including the first saddle-point avoidance result for PnP methods, as well as convergence to stationarity under annealed noise. Third, we provide numerical validation on natural RGB images and brain MRI, demonstrating substantial improvements over prior PnP methods.

\section{Background}
\label{sec:background}

\medskip\noindent
\textbf{Imaging Inverse Problems.} Imaging inverse problems aim to recover an unknown signal $\x \in \mathbb{R}^n$ from noisy measurements $\y \in \mathbb{R}^m$ modeled as
\begin{equation}
    \y = \vec{A}(\x) + \e \;,
\end{equation}
where $\vec{A}:\mathbb{R}^n \to \mathbb{R}^m$ denotes the forward operator and $\e\sim \mathcal{N}(\vec{0}, \eta^2 \vec{I})$ denotes additive Gaussian measurement noise with noise level $\eta$.

A common approach is to formulate reconstruction as a regularized optimization problem
\begin{equation} \label{eq:begin_inv_problem}
\widehat{\x} \in \argmin_{\x \in \mathbb{R}^n} \; f(\x) \quad \text{with} \quad f(\x) = g(\x) + h(\x) \;,
\end{equation}
where $g(\x)$ is a data-fidelity term that enforces consistency with the observed measurements $\y$, and $h(\x)$ is a regularizer encoding prior knowledge about $\x$.
From a Bayesian perspective, \eqref{eq:begin_inv_problem} is a maximum a posteriori (MAP) estimator when
\begin{equation}
\label{eq:g_and_h}
    g(\x) = -\log p(\y|\x) \;, \qquad
    h(\x) = -\log p(\x) \;,
\end{equation}
where $p(\y|\x)$ denotes the likelihood model and $p(\x)$ is the prior distribution.
In many imaging systems, $\e$ is modeled as additive white Gaussian noise, in which case the data-fidelity term reduces to a squared $\ell_2$ term $g(\x) = \tfrac{1}{2}\|\y - \vec{A}(\x)\|_2^2$.

\medskip\noindent
\textbf{Traditional PnP Reconstruction.} Proximal splitting algorithms~\cite{Parikh.Boyd2014} are widely used to solve optimization problems of the form \eqref{eq:begin_inv_problem}, particularly when the data-fidelity term $g$ or the regularizer $h$ is nonsmooth.
A central concept underlying these methods is the scaled proximal operator associated with $h$, defined as
\begin{equation} \label{eq:prox_def}
    \prox_{\gamma h}(\z)
    \coloneqq \argmin_{\x \in \mathbb{R}^n}
    \left\{ \tfrac{1}{2}\|\x - \z\|_2^2 + \gamma h(\x) \right\} \;,
\end{equation}
where $\gamma > 0$ is the scaling parameter.
From a probabilistic perspective, the proximal operator can be interpreted as a maximum a posteriori (MAP) estimator for an additive white Gaussian noise (AWGN) denoising problem, with $h(\x)$ corresponding to the negative log-prior.

PnP methods leverage this interpretation by replacing the proximal operator $\prox_{\gamma h}$ with a general image denoiser $\mathsf{D}:\mathbb{R}^n\to\mathbb{R}^n$ within iterative optimization algorithms, while keeping the data-fidelity update unchanged.

A representative example is PnP-ADMM~\cite{Venkatakrishnan.etal2013, Sreehari.etal2016,  Chan.etal2016}, which replaces $\mathrm{prox}_{\gamma h}$ in ADMM with a pretrained image denoiser $\mathsf{D}_{\theta}$, where $\theta$ denotes the denoiser parameters, resulting in the iterates
\begin{subequations} \label{eq:admm}
\begin{align}
    \label{subeq:admm1}
    \x_k &\leftarrow \mathrm{prox}_{\gamma g} \left( \z_{k-1} - \s_{k-1} \right)\\
    \label{subeq:admm2}
    \z_k &\leftarrow \mathsf{D}_{\theta}\left( \x_k + \s_{k-1} ; \sigma_k\right)\\
    \label{subeq:admm3}
    \s_k &\leftarrow \s_{k-1} + \x_k - \z_k \;,
\end{align}
\end{subequations}
where $\sigma_k$ denotes the conditioning noise level at the $k$-th iteration that controls the denoising strength.
Related deterministic PnP formulations arise in regularization by denoising (RED)~\cite{romano2017RED} and half-quadratic splitting (HQS)–based methods such as deep plug-and-play image restoration (DPIR)~\cite{zhang2021dpir}, where denoisers are used as priors within iterative schemes such as gradient-based and splitting-based methods. Another related framework is the multi-scale energy method~\cite{chand2024multiscale}, which learns an explicit energy functional spanning a broad range of noise levels. The resulting energy is minimized deterministically to compute a MAP estimate, with theory guaranteeing monotone convergence to a stationary point.

\medskip\noindent
\textbf{Stochastic PnP Reconstruction.} Stochastic variants of PnP reconstruction have already been explored, with
several approaches~\cite{tang2020fast, sun2021scalable, wu2019online, sun2020asyncred, liu2021sgdnet, sun2018plugin, sun2019onlinePnP} introducing stochasticity to improve computational efficiency through mini-batch approximations.

More closely related to our setting, stochastic denoising regularization (SNORE)~\cite{renaud2024snore} introduces stochasticity into PnP by injecting Gaussian noise into the denoiser input to reduce the mismatch between intermediate PnP iterates and the noise statistics assumed during denoiser training.
In particular, it constructs an explicit stochastic regularizer by applying the denoiser to perturbed iterates and establishes convergence guarantees to critical points of a corresponding smoothed objective. For example, such stochastic updates in PnP-ADMM can be written as
\begin{equation}
\label{eq:snore}
    \z_k
    =
    \mathsf{D}_{\theta}\big(
        \x_k + \s_{k-1} + \sigma^{\text{inject}}\vec{n};
        \sigma^{\text{inject}}
    \big) \;,
\end{equation}
where $\vec{n}\sim\mathcal{N}(\vec{0},\vec{I})$ and $\sigma^{\text{inject}}$ controls the injected noise level. 
While this perspective improves theoretical stability and partially mitigates distribution mismatch, empirical improvements over deterministic PnP remain limited. 
For example, in the deblurring experiments of \cite[Table 1]{renaud2024snore}, SNORE achieves PSNR values that are roughly $0.6$–$0.7$\,dB lower than the corresponding deterministic PnP method across noise levels, while occasionally improving perceptual metrics. 
This leaves open how stochastic PnP methods should be designed to effectively address inverse problems.

\medskip\noindent
\textbf{PnP Diffusion Models.} Recent approaches to combining diffusion priors with data-consistency updates generally fall into two categories: (a) integrating data fidelity into the diffusion sampling process, and (b) incorporating generative sampling steps within PnP reconstruction algorithms.

The first category modifies the reverse diffusion process by injecting measurement consistency directly into the sampling trajectory~\cite{kawar2022ddrm, chung2023dps, zhu2023DiffPIR, xu2024provably_dps_pnp, laumont2022pnpula, bouman2023generativePnP}. For example, diffusion posterior sampling (DPS)~\cite{chung2023dps} approximates the likelihood gradient using the denoised estimate $\hat{\x}_\theta(\x_t)$:
\begin{equation}
\label{eq:dps}
\nabla_{\x_t}\log p(\y|\x_t)
\approx
-\frac{1}{\eta^2}\nabla_{\x_t} \| \y - \vec{A}(\hat{\x}_\theta(\x_t))\|^2_2 \;,
\end{equation}
where $\eta$ is the measurement noise level.
Similarly, denoising diffusion models for plug-and-play image restoration (DiffPIR)~\cite{zhu2023DiffPIR} apply proximal data-consistency steps to the clean image estimate before continuing the reverse diffusion iteration. A related approach~\cite{bendel2025solving} re-noises to keep the denoiser's input close to the white-Gaussian regime it
was trained on, but does so with colored noise inside the reverse-diffusion trajectory (running several denoiser evaluations per diffusion step).
The second category~\cite{sun2024provablePMC, wu2024principledPnP, coeurdoux2024pnpsplitgibbssampler, faye2024REDbayesian, gulle2025consistency} instead replaces the deterministic denoiser in PnP with generative sampling procedures, thereby turning PnP into a stochastic posterior sampling scheme rather than an optimization algorithm.

In contrast to both approaches, our framework retains the \emph{optimization-centric} formulation of PnP for solving~\eqref{eq:begin_inv_problem}, rather than recasting it as a generative \emph{sampling} procedure. We use the diffusion model purely as a denoiser operating across a wide range of noise levels, without embedding generative steps within PnP.
Moreover, we introduce a noise injection mechanism that mitigates the mismatch between intermediate iterates and the denoiser’s training distribution, while enabling escape from saddle points.

\section{Proposed Methods}

\subsection{Score Adaptation for PnP}
\label{sec:score_adaptation}

We first describe how pre-trained SBDM networks are adapted as denoisers within PnP algorithms.
This approach enables solving \eqref{eq:begin_inv_problem} using a score-based regularizer within a PnP algorithm, without requiring reverse diffusion iterations.

\medskip\noindent
\textbf{Relating Score to Denoising.} Tweedie's formula~\cite{efron2011tweedie} links the MMSE denoiser to the score function. To adapt various types of SBDMs, we define a general noise perturbation scheme
\begin{equation}\label{eq:general_noise_addition}
    \x_{c\sigma} = c(\x + \sigma \vec{n})
    \quad
    \vec{n} \sim \normpdf(\vec{0}, \vec{I}) \;,
\end{equation}
where $c \in \mathbb{R}^+$ is a scale factor and $\sigma$ is the noise level. Tweedie's formula gives the general score-based denoising template
\begin{equation}\label{eq:general_tweedie}
\mathsf{D}_{\theta}(\x;\sigma)= \x + c\sigma^2 \nabla \log p_{c \sigma}(c \x) \;,
\end{equation}
where $p_{c \sigma}$ is the density of the scaled noisy observation. As $c \rightarrow 1$ and $\sigma \rightarrow 0$, this approaches the noise-free image distribution $p(\x)$. We apply this template to two common SBDM classes:

\medskip\noindent
\textbf{Variance-Exploding SBDMs~\cite{song2021sde}.}  Variance-exploding (VE) diffusion models are trained using the noise corruption process $\x_{\sigma_t} = \x + \sigma_t \vec{n}$.
Matching this to our general noise perturbation scheme in \eqref{eq:general_noise_addition} yields $c=1$ and $\sigma=\sigma_t$.
We can then map the pre-trained VE diffusion model to the PnP denoiser using \eqref{eq:general_tweedie}
\begin{equation} \label{eq:ve_mmse_score}
\begin{aligned}
    \mathsf{D}_{\theta}(\x;t) &= \x + \sigma_{t}^2 \s^{\text{VE}}_{\theta}(\x, t) \;,
\end{aligned}
\end{equation}
where $\s^{\text{VE}}_{\theta}(\x, t)$ is the time step-conditional VE score network that approximates $\nabla \log p_{\sigma_t}(\x)$.

\medskip\noindent
\textbf{Variance-Preserving SBDMs~\cite{ho_NEURIPS2020_ddpm}.}
Variance-preserving (VP) diffusion models are trained using the noise corruption process
\begin{equation}\label{eq:vp_noise_addition}
    \x_{\bar{\alpha}_t} = \sqrt{\bar{\alpha}_t}\left(\x + \sqrt{\frac{1-\bar{\alpha}_t}{\bar{\alpha}_t}}\vec{n}\right) \;,
\end{equation}
where $\bar{\alpha}_t = \prod_{s=1}^t \alpha_s$, and $\alpha_t$ is chosen to ensure $\x_{\bar{\alpha}_0}$ follows desired probability distribution and $\x_{\bar{\alpha}_{T}}$ follows standard normal distribution.
Matching this to the noise perturbation scheme in \eqref{eq:general_noise_addition} yields $c=\sqrt{\bar{\alpha}_t}$ and $\sigma = \sqrt{\frac{1-\bar{\alpha}_t}{\bar{\alpha}_t}}$.
The VP diffusion model can then be mapped to PnP denoising as
\begin{equation}\label{eq:vp_general_matching}
\begin{aligned}
\mathsf{D}_{\theta}(\x;t) &= \x + \frac{1 - \bar{\alpha}_{t}}{\sqrt{\bar{\alpha}_{t}}} \s^{\text{VP}}_{\theta}(\sqrt{\bar{\alpha}_{t}} \, \x, t) \;,
\end{aligned}
\end{equation}
where $\s^{\text{VP}}_{\theta}(\x, t)$ is the time-conditional VP score network that approximates $\nabla \log p_{\sqrt{\bar{\alpha}_{t}}\sigma}(\x)$ and $t \in [1, T]$ parameterizes the noise level.

\begin{figure*}[t]
\centering
\begin{minipage}[t]{0.29\textwidth}
\begin{algorithm}[H]
\setstretch{1.1} %
\caption{SGPnP-ADMM}
\begin{algorithmic}[1]\vspace{-4pt}
\Require $\y, \vec{A}, \gamma_k, \sigma^{\text{inject}}_k, \sigma^{\text{cond}}_k$
\State $\x_0 \leftarrow \vec{A}^{H}\y$, $\s_0 \leftarrow \vec{0}$
\For{$k = 0 \text{ to } K - 1$}
        \State \(\z_{k} \leftarrow \operatorname{prox}_{\gamma_k g}(\x_{k} - \s_{k})\)
        \State \(\vec{n} \sim \mathcal{N}(\vec{0}, \vec{I})\)
        \State \(\x_{\textsf{input}} \leftarrow (\s_{k} + \z_{k}) + \sigma_k^{\text{inject}}\vec{n}\)
        \State \(\x_{k+1} \leftarrow \mathsf{D}_{\theta}(\x_{\textsf{input}}\,;\,\sigma_k^{\text{cond}})\)
        \State \(\s_{k+1} \leftarrow \s_{k} + \z_{k} - \x_{k+1}\)
\EndFor
\State \textbf{return} $\x_{K}$
\end{algorithmic}
\label{alg:stochastic_pnpadmm}
\end{algorithm}
\end{minipage}\hfill
\begin{minipage}[t]{0.335\textwidth}
\begin{algorithm}[H]
\setstretch{1.1} %
\caption{SGPnP-PGM}
\begin{algorithmic}[1]\vspace{-4pt}
\Require $\y, \vec{A}, \tau_k, \gamma_k, \sigma^{\text{inject}}_k, \sigma^{\text{cond}}_k$
\State $\x_0 \leftarrow \vec{A}^{H}\y$
\For{$k = 0 \text{ to } K - 1$}
    \State \(\vec{n} \sim \mathcal{N}(\vec{0}, \vec{I})\)
    \State \(\x_{\textsf{input}} \leftarrow \x_k + \sigma_k^{\text{inject}} \vec{n}\)
    \State \(\nabla h(\x_{k}) = \x_{k} - \mathsf{D}_{\theta}(\x_{\textsf{input}}\,;\,\sigma_k^{\text{cond}})\)
    \State \(\z_k \leftarrow \operatorname{prox}_{\gamma_k g}(\x_k)\)
    \State \(\x_{k+1} \leftarrow \z_k - \tau_k \nabla h(\x_k)\)
\EndFor
\State \textbf{return} $\x_{K}$
\end{algorithmic}
\label{alg:stochastic_pgm}
\end{algorithm}
\end{minipage}
\begin{minipage}[t]{0.36\textwidth}
\begin{algorithm}[H]
\setstretch{1.1} %
\caption{SGPnP-RED}
\begin{algorithmic}[1]\vspace{-4pt}
\Require $\y, \vec{A},  \tau_k, \gamma_k, \sigma^{\text{inject}}_k, \sigma^{\text{cond}}_k$
\State $\x_0 \leftarrow \vec{A}^{H}\y$
\For{$k = 0 \text{ to } K - 1$}
    \State \(\vec{n} \sim \mathcal{N}(\vec{0}, \vec{I})\)
    \State \(\x_{\textsf{input}} \leftarrow \x_k + \sigma_k^{\text{inject}} \vec{n}\)
    \State \(\nabla g(\x_{k}) = \vec{A}^{H}\left(\vec{A}\x_{k} - \y\right)\)
    \State \(\nabla h(\x_{k}) = \tau_k (\x_{k} - \mathsf{D}_{\theta}(\x_{\textsf{input}}\,;\,\sigma_k^{\text{cond}}))\)
    \State \(\x_{k+1} \leftarrow \x_k - \gamma_k  (\nabla g(\x_k) + \nabla h(\x_k))\)
\EndFor
\State \textbf{return} $\x_{K}$
\end{algorithmic}
\label{alg:stochastic_red}
\end{algorithm}
\end{minipage}\hfill
\end{figure*}

\medskip\noindent
\textbf{Parameter Matching.}
A challenge in applying a pretrained off-the-shelf SBDM in PnP is the mismatch between the two parameterizations: many pretrained diffusion models are conditioned on time steps $t \in \{1,\dots,T\}$ (i.e., a discrete time-indexed parameterization), whereas PnP typically uses a denoiser parameterized by a continuous noise level $\sigma$.
To bridge this, we must identify the specific time step $t$ that corresponds to the query noise level $\sigma$.
Let $\rho(t)$ denote the effective noise schedule of the SBDM (e.g., $\rho(t) = \sigma_t$ for VE or $\rho(t) = \sqrt{(1-\bar{\alpha}_t)/\bar{\alpha}_t}$ for VP).
Our goal is to compute the inverse mapping $t = \rho^{-1}(\sigma)$.

Since $\rho(t)$ is only defined at integer indices $t \in \{1, \dots, T\}$, we construct a continuous approximation by linearly interpolating the schedule $\{\rho(t)\}_{t=1}^T$ onto a high-resolution grid.
Formally, given a PnP noise level $\sigma$, we determine the continuous time parameter $t^*$ by numerically inverting the interpolated schedule.
This ensures that the score network is conditioned on the exact noise variance required by the PnP iteration (we report the denoising performance at fractional time steps between the discrete training steps in Appendix~\ref{app:more_visualization}-\ref{sub:interp_fractional} (Table~\ref{tab:interp_fractional})).
With $t^*$ determined, the denoiser parameters are fully specified by setting $t = t^*$ and deriving the scaling constant $c$ from the corresponding noise perturbation model (VE or VP).

\subsection{Proposed Generative PnP Framework}
\label{sec:stochastic_pnp}

The score-adaptation procedure enables pre-trained SBDMs to serve as noise-conditioned denoisers within PnP algorithms.
We now introduce two score-based PnP frameworks that integrate these pre-trained SBDMs into the PnP denoising step: the score-based deterministic PnP (SDPnP) and the stochastic generative PnP (SGPnP).

\medskip\noindent
\textbf{Deterministic prior update.}
PnP algorithms typically consist of alternating data-consistency and
prior-enforcement steps.
A generic template can be written as
\begin{equation}
\label{eq:deterministic_pnp_template}
\z_k = \mathsf{DC}_k(\x_k;\y) \;,
\qquad
\x_{k+1} = \mathsf{D}_\theta(\z_k;\sigma_{k}^{\text{cond}}) \;,
\end{equation}
where $\x_k$ denotes the current reconstruction estimate at iteration $k$, $\y$ represents the observed measurements, and $\z_k$ is the intermediate variable obtained after the data-consistency (DC) update.
The operator $\mathsf{DC}_k$ depends on the chosen PnP algorithm (such as ADMM, HQS/DPIR, or PGM), and the prior step incorporates learned
image statistics through denoising $\mathsf{D}_\theta(\z_k;\sigma_{k}^{\text{cond}})$, where $\sigma_{k}^{\text{cond}}$ controls the denoising level.

By applying the score-adaptation formulas derived in
Section~\ref{sec:score_adaptation}, the denoising update in \eqref{eq:deterministic_pnp_template} can be implemented
directly using a pretrained SBDM denoiser.
Specifically, depending on the diffusion training formulation,
$\mathsf{D}_{\theta}$ may correspond to the VE-based denoiser
in~\eqref{eq:ve_mmse_score} or the VP-based denoiser
in~\eqref{eq:vp_general_matching}.
We refer to the PnP algorithm with a score-based prior as score-based deterministic PnP (SDPnP).

\medskip\noindent
\textbf{Stochastic prior update.} 
While the deterministic prior update directly applies the SBDM denoiser, the intermediate iterate $\z_k$ produced by the data-consistency step is not, in general, distributed like the denoiser's training input. In particular, $\z_k$ may not be well represented as a sample from the image prior corrupted by Gaussian noise at the prescribed noise conditioning level.
To better leverage the broad noise-level conditioning of SBDMs, we introduce an explicit re-noising step before denoising, forming our stochastic generative PnP (SGPnP) framework. Specifically,
\begin{equation}
\label{eq:stochastic_pnp_template}
\z_k = \mathsf{DC}_k(\x_k;\y) \;,
\qquad
\x_{k+1} =
\mathsf{D}_\theta(\z_k+\textcolor{blue}{\sigma_{k}^{\text{inject}}}\vec{n};
\sigma_{k}^{\text{cond}}) \;,
\end{equation}
where $\vec{n}\sim\mathcal{N}(\mathbf{0},\mathbf{I})$.
Here, $\textcolor{blue}{\sigma_{k}^{\text{inject}}}$ controls the magnitude of the stochastic perturbation, while $\sigma_{k}^{\text{cond}}$ determines the
denoising strength of the SBDM prior.
We decouple the injected and conditioning noise levels, allowing the denoiser strength to account for both the injected perturbation and residual artifacts in the intermediate iterate. As a result, SGPnP reformulates the deterministic denoising operation in PnP as a controlled stochastic transition, improving robustness in imaging
inverse problems (see Algorithms~\ref{alg:stochastic_pnpadmm}--\ref{alg:stochastic_red}).
We provide a detailed comparison with a related stochastic PnP framework in Appendix~\ref{app:difference_with_others}.

We specify both noise levels with a log-uniform annealing schedule, $\sigma_k=\sigma_0\,(\sigma_{K-1}/\sigma_0)^{k/(K-1)}$ for $k=0,\dots,K-1$, with endpoints and step sizes given in Appendix~\ref{app:hyperparameters}.
We parameterize the two levels by their ratio $r=\sigma_k^{\text{cond}}/\sigma_k^{\text{inject}}$, where $r=1$ recovers the matched setting of SNORE~\cite{renaud2024snore} and $r>1$ conditions the denoiser above the injected level. We use $r=1$ as the default on FFHQ and $r>1$ on the measurement-rich CS-MRI problem. Selection guidance and a full sweep are reported in
Appendix~\ref{app:more_visualization}-\ref{sub:selecting_ratio} where it is demonstrated that a coarse check of one larger value suffices in practice.

\section{Theoretical Analysis}
\label{sec:theory_escape}

We now analyze SGPnP in Algorithm~\ref{alg:stochastic_red} with injected Gaussian perturbations at the denoiser input and establish two theoretical guarantees.
First, we show that the resulting stochasticity enables escape from strict saddle points under suitable assumptions. Exploiting results on stochastic gradients in non-convex optimization~\cite{daneshmand2018escaping}, we prove that under a variance-preservation condition on the denoiser, the injected perturbations induce sufficient drift along directions of negative curvature to escape strict saddle points of the smoothed objective. Second, we analyze an annealed noise schedule and establish convergence to stationarity of the exact (un-smoothed) objective in expectation.
The analysis covers the matched-noise instance $\sigma_k^{\text{inject}}=\sigma_k^{\text{cond}}=\sigma_k$; accordingly, $\mathsf{D}_\theta(\cdot)$ below denotes $\mathsf{D}_\theta(\cdot\,;\,\sigma)$. The decoupled setting $\sigma_k^{\text{inject}}\neq\sigma_k^{\text{cond}}$ used in some experiments (e.g., Table~\ref{tab:sigma_inject_sweep}) is an empirical extension not covered by the theory.

Let $g(\x)$ denote the data-fidelity term and let $p(\x)$ be a clean-image prior.
Following \cite{renaud2024snore}, we define the Gaussian-smoothed prior $p_\sigma(\x) = (p * G_\sigma)(\x)$, where $G_\sigma$ is the isotropic Gaussian smoothing kernel corresponding to the density of $\mathcal{N}(\vec{0}, \sigma^2 \vec{I})$. The associated stochastic regularizer is
\begin{equation}
\label{eq:h_sigma_def}
h_\sigma(\x)
\coloneqq
-\mathbb E_{\vec{n}\sim\mathcal N(\vec{0}, \vec{I})}
\big[\log p_\sigma(\x+\sigma \vec{n})\big] \;.
\end{equation}
We define the stochastic vector field 
\begin{equation}
U_\sigma(\x,\vec{n})\coloneqq\sigma^{-2}(\x-\mathsf D_\theta(\x+\sigma \vec{n})) \;,
\end{equation}
where $\vec{n} \sim \mathcal{N}(\vec{0}, \vec{I})$ introduces stochasticity through the injected perturbation. The composite objective is $f_\sigma(\x) = g(\x)+h_\sigma(\x)$ and the SGPnP iteration is
\begin{align}
\label{eq:iter_final_clean}
\x_{k+1}
= 
\x_k
-
\gamma_k
\left(
\nabla g(\x_k)
+
U_{\sigma_k}(\x_k,\vec{n}_k)
\right) \;,
\end{align}
where $\vec{n}_k\sim\mathcal N(\vec{0}, \vec{I})$.
We define the implicit optimization noise as the deviation of the stochastic vector field from its expectation:
\begin{equation}
\label{eq:w_k}
\w_k
\coloneqq
U_{\sigma_k}(\x_k,\vec{n}_k) - \mathbb E_{\vec{n}\sim \mathcal{N}(\vec{0}, \vec{I})}[U_{\sigma_k}(\x_k,\vec{n})] \;.
\end{equation}

\subsection{Strict saddle avoidance}

\begin{definition}
\label{def:strict_saddle}
A critical point $\x^\dagger$ of $f_\sigma$ is a \emph{strict saddle} if $\nabla f_\sigma(\x^\dagger)=0$ and $\nabla^2 f_\sigma(\x^\dagger)$ has at least one strictly negative eigenvalue~\cite{pemantle1990nonconvergence}.
\end{definition}

\begin{assumption}
\label{ass:smoothness}
For each fixed $\sigma>0$, the data-fidelity term $g$ and the regularizer $h_\sigma$ have Lipschitz gradients (with constants $L_g$ and $L_h(\sigma)$) and Lipschitz Hessians (with constants $\rho_g$ and $\rho_h$, where $\rho_h$ may depend on $\sigma$). Consequently, the composite objective $f_\sigma$ has an $L(\sigma)$-Lipschitz gradient and a $\rho$-Lipschitz Hessian, where $L(\sigma) = L_g + L_h(\sigma)$ and $\rho = \rho_g + \rho_h$. 
We write $L$ for $L(\sigma)$ when $\sigma$ is fixed. 
Moreover, $f_{\sigma}$ has a bounded gradient and is bounded below,
i.e., $f_\sigma^\star := \inf_{\x} f_\sigma(\x) > -\infty$.
\end{assumption}

A function $g$ has an $L_g$-Lipschitz gradient if $\|\nabla g(\x) - \nabla g(\y)\| \leq L_g\|\x - \y\|$ for all $\x, \y \in \mathbb{R}^n$, and a $\rho_g$-Lipschitz Hessian if $\|\nabla^2 g(\x) - \nabla^2 g(\y)\| \leq \rho_g\|\x - \y\|$. Assumption~\ref{ass:smoothness} imposes standard regularity conditions common in both computational imaging~\cite{laumont2022pnpula, sun2024provablePMC, sun2021scalable, sun2020asyncred, sun2019onlinePnP} and non-convex optimization~\cite{carmon2017convex, agarwal2017finding, marumo2024parameter, li2023restarted}. These conditions naturally hold for linear inverse problems with additive Gaussian noise, where $g$ is a quadratic function and its Hessian is constant. While the Hessian condition on $h_\sigma$ might appear restrictive, it is justified in our setting since $h_\sigma$ is Gaussian-smoothed, and such smoothing improves higher-order regularity.

\begin{assumption}
\label{ass:unbiased}
The pretrained denoiser is an MMSE denoiser, so that for all $\x$, 
\[\mathbb E_{\vec{n}\sim \mathcal{N}(\vec{0}, \vec{I})}[U_\sigma(\x,\vec{n})]=\nabla h_\sigma(\x).\]
\end{assumption}
Assumption~\ref{ass:unbiased} is standard in the analysis of score-based denoisers. Such results rely on Tweedie’s formula, which links the score function to the MMSE denoiser~\cite{vincent2011connection, song2021sde, renaud2024snore}.

\begin{assumption}
\label{ass:denoiser_variance}
For any state $\x \in \mathbb{R}^n$, let $\vec{v}_{\x}$ be the eigenvector
associated with the minimum eigenvalue of $\nabla^2 f_\sigma(\x)$.
We assume the stochastic gradient is almost surely bounded,
$\|U_\sigma(\x,\vec{n})\|\le B$, and has upper-bounded variance, i.e.,
$\mathrm{Var}_{\vec{n}}(U_\sigma(\x, \vec{n})) \le V_{\max}(\sigma)$, and that there
exists a constant $c > 0$ such that
\begin{equation}
\mathrm{Var}_{\vec{n} \sim \mathcal{N}(0, I)} \big( \vec{v}_{\x}^\top \mathsf{D}_\theta(\x+\sigma \vec{n}) \big) \ge c \;.
\end{equation}
\end{assumption}
Assumption~\ref{ass:denoiser_variance} ensures that the injected noise has a non-degenerate component along directions of negative curvature, which is essential for escaping strict saddle points. 
This condition is formally justified by Lemma~\ref{lem:saddle_point_escape} in Appendix~\ref{app:theorydetail}, which shows that when the data-fidelity term $g$ is convex (as in our experiments), the variance lower bound holds at any strict saddle point with $c = \sigma^2$.

\begin{theorem}
\label{thm:escape_saddle}
Run the SGPnP iteration in~\eqref{eq:iter_final_clean} under Assumptions~\ref{ass:smoothness}--\ref{ass:denoiser_variance}. Then, for any $\delta \in (0, 1)$, there exists a stepsize schedule $\{\gamma_k\}_{k\geq0}$ and number of iterations $K$ such that, with probability at least $1-\delta$, the iterates avoid strict saddle points of $f_\sigma$.
\end{theorem}

The proof is provided in Appendix~\ref{app:theorydetail}, where we also detail the exact parameter schedule required to rigorously bound the probability of successfully escaping strict saddle points.
To the best of our knowledge, this is the first formal analysis of saddle-point escape in the PnP literature. The key insight is that noise injected at the denoiser input drives the iterates away from directions of negative curvature. We isolate the overall effect of noise injection on reconstruction performance in Table~\ref{table:comparison_table2} (SDPnP vs.\ SGPnP).

\subsection{Convergence under annealed noise schedule}

We next analyze SGPnP under a decreasing noise schedule and establish convergence to stationarity of the exact, un-smoothed objective $f_0$ in expectation. For a fixed noise level $\sigma$, we define the set of critical points as $S_\sigma \coloneqq \{\x : \nabla f_\sigma(\x) = \bm{0}\}$. To ensure asymptotic consistency as the injected noise vanishes, we require a regularity condition.

\begin{assumption}
\label{ass:anneal_conv}
The function $h_0(\x) = -\log p(\x)$, is continuously differentiable and the objective $f_0(\x) = g(\x) + h_0(\x)$ is coercive. Additionally, we assume the iterates produced by SGPnP remain within a compact set $\mathcal{K}$, on which
\begin{equation}
\lim_{\sigma \rightarrow 0} \sup_{\x\in \mathcal{K}}\|\nabla h_\sigma(\x)-\nabla h_0(\x)\| = 0 \;.
\end{equation}
\end{assumption}
Assumption~\ref{ass:anneal_conv} ensures that the gradients of $f_\sigma$ converge uniformly to those of the exact objective $f_0$ on compact sets. Such behavior holds for Gaussian smoothing under mild regularity conditions on the prior density $p$, including smoothness.  This assumption is consistent with the Gaussian smoothing framework underlying score-based diffusion models, where the smoothed densities $p_\sigma$ approximate the true data distribution $p$ as $\sigma \to 0$. Here, we make this convergence explicit at the level of gradients to enable analysis of optimization trajectories.
As a result, the critical points of $f_\sigma$ approximate those of $f_0$ as $\sigma \to 0$. 
For analysis of the annealing process, we consider a staged regime in which, for each noise level $\sigma_k$, the iterates approach a neighborhood of the critical set of the smoothed objective (Lemma~\ref{lem:neighborhood}) before $\sigma$ is further reduced.

\begin{theorem}
\label{thm:annealed}
Suppose Assumptions~\ref{ass:smoothness}--\ref{ass:anneal_conv} hold.
Consider a schedule $\{(\sigma_k,\gamma_k)\}_{k\ge 0}$ with
$\sigma_k\to 0$ and $\gamma_k\to 0$ as $k\to\infty$, 
satisfying $\gamma_k\le 1/L_k$ and $\gamma_k L_k V_{\max}(\sigma_k)\to 0$, where $L_k\coloneqq L(\sigma_k)$.
At each stage $k$, run the inner loop 
until the returned iterate $\hat\x_k$ satisfies
$\mathbb{E}\|\nabla f_{\sigma_k}(\hat\x_k)\|^2 \le 2\,L_k\gamma_k V_{\max}(\sigma_k)$,
which is attainable in a finite $I_k$.
Then
\begin{equation}
\mathbb{E}\!\left[\|\nabla f_0(\hat\x_k)\|^2\right]\;\longrightarrow\;0
\quad\text{as}\quad k\to\infty,
\end{equation}
where $f_0(\x)=g(\x)+h_0(\x)$ is the exact objective.
\end{theorem}

The proof is provided in Appendix~\ref{app:theorydetail}. Theorem~\ref{thm:escape_saddle} and Theorem~\ref{thm:annealed} together characterize optimization behavior of SGPnP. At fixed noise, Theorem~\ref{thm:escape_saddle} establishes avoidance of strict saddle points. Under annealing, Theorem~\ref{thm:annealed} establishes convergence to stationary points of the original un-smoothed objective in expectation.

\begin{table*}
    \centering
    \small
    \caption{
    Quantitative comparison between several PnP baselines and diffusion-based solvers on FFHQ (inpainting, $4\times$ super-resolution, deblurring) and fastMRI ($4\times$ CS-MRI).
    \hlgreen{\textbf{Best values}} are highlighted for each metric and inverse problem.
    SGPnP-PGM consistently improves upon conventional PnP methods and is comparable to diffusion-based solvers.
    }
    
    \vspace{0.15cm}
    \renewcommand{\arraystretch}{0.7}
    \resizebox{0.6\textwidth}{!}{
    \begin{tabular}{@{}p{2.cm}p{0.1cm}p{0.5cm}p{0.5cm}p{0.5cm}p{0.5cm}@{}p{0.2cm}@{}p{0.5cm}@{}}
    \toprule
    \textbf{Testing data} &  & \multicolumn{1}{c}{\text{Input}} & \multicolumn{1}{c}{\text{DPIR}} & \multicolumn{1}{c}{\text{SNORE}}  & \multicolumn{1}{c}{\text{DPS}}  & \multicolumn{1}{c}{\text{DiffPIR}}  &  \multicolumn{1}{c}{\text{SGPnP}} \\
    \cmidrule{1-8}\\ \noalign{\vskip -1.9ex}
        \multirow{3}{*}{\textbf{Inpainting}} & \multicolumn{1}{c}{PSNR$\uparrow$}  & \multicolumn{1}{c}{$18.17$} & \multicolumn{1}{c}{$18.42$} & \multicolumn{1}{c}{$18.50$}  &
      \multicolumn{1}{c}{$23.69$} & \multicolumn{1}{c}{$25.09$} & \multicolumn{1}{c}{\hlgreen{$\mathbf{25.21}$}} \\[+.95ex]
     & \multicolumn{1}{c}{SSIM$\uparrow$} & \multicolumn{1}{c}{$0.766$} & \multicolumn{1}{c}{$0.797$} & \multicolumn{1}{c}{$0.799$}   &
      \multicolumn{1}{c}{$0.810$} & \multicolumn{1}{c}{$0.860$}  & \multicolumn{1}{c}{\hlgreen{$\mathbf{0.874}$}} \\[+.75ex]
      & \multicolumn{1}{c}{LPIPS$\downarrow$} & \multicolumn{1}{c}{$0.289$} & \multicolumn{1}{c}{$0.264$} & \multicolumn{1}{c}{$0.250$} &
      \multicolumn{1}{c}{$0.176$} & \multicolumn{1}{c}{$0.115$} & \multicolumn{1}{c}{\hlgreen{$\mathbf{0.108}$}} \\[+.5ex]  
      \cdashline{1-8} \\ \noalign{\vskip -1.ex}
    \multirow{3}{*}{\textbf{Super-resolution}} & \multicolumn{1}{c}{PSNR$\uparrow$}  & \multicolumn{1}{c}{$24.61$} & \multicolumn{1}{c}{$26.54$} & \multicolumn{1}{c}{$26.40$}  &
      \multicolumn{1}{c}{$28.66$} & \multicolumn{1}{c}{$29.69$} & \multicolumn{1}{c}{\hlgreen{$\mathbf{29.90}$}} \\[+.95ex]
     & \multicolumn{1}{c}{SSIM$\uparrow$} & \multicolumn{1}{c}{$0.778$} & \multicolumn{1}{c}{$0.833$} & \multicolumn{1}{c}{$0.837$} &
      \multicolumn{1}{c}{$0.839$} & \multicolumn{1}{c}{$0.854$} & \multicolumn{1}{c}{\hlgreen{$\mathbf{0.874}$}} \\[+.75ex]
      & \multicolumn{1}{c}{LPIPS$\downarrow$} & \multicolumn{1}{c}{$0.305$} & \multicolumn{1}{c}{$0.255$} & \multicolumn{1}{c}{$0.240$} &
      \multicolumn{1}{c}{\hlgreen{$\mathbf{0.140}$}} & \multicolumn{1}{c}{$0.194$} & \multicolumn{1}{c}{$0.194$} \\[+.5ex]    
      \cdashline{1-8} \\ \noalign{\vskip -1.ex}
      \multirow{3}{*}{\textbf{Deblurring}} & \multicolumn{1}{c}{PSNR$\uparrow$}  & \multicolumn{1}{c}{$23.69$} & \multicolumn{1}{c}{$33.42$} & \multicolumn{1}{c}{$33.18$} &
      \multicolumn{1}{c}{$33.92$} & \multicolumn{1}{c}{$33.48$} & \multicolumn{1}{c}{\hlgreen{$\mathbf{34.52}$}}  \\[+.95ex]
     & \multicolumn{1}{c}{SSIM$\uparrow$} & \multicolumn{1}{c}{$0.703$} & \multicolumn{1}{c}{$0.923$} & \multicolumn{1}{c}{$0.932$} &
      \multicolumn{1}{c}{$0.915$} & \multicolumn{1}{c}{$0.913$}  & \multicolumn{1}{c}{\hlgreen{$\mathbf{0.935}$}} \\[+.75ex]
      & \multicolumn{1}{c}{LPIPS$\downarrow$} & \multicolumn{1}{c}{$0.326$} & \multicolumn{1}{c}{$0.123$} & \multicolumn{1}{c}{$0.085$} &
      \multicolumn{1}{c}{$0.076$} & \multicolumn{1}{c}{$0.110$}  & \multicolumn{1}{c}{\hlgreen{$\mathbf{0.068}$}} \\
      \bottomrule \\ \noalign{\vskip -1.ex}
    \multirow{3}{*}{\textbf{CS-MRI}} & \multicolumn{1}{c}{PSNR$\uparrow$}  & \multicolumn{1}{c}{$22.43$} & \multicolumn{1}{c}{$26.93$} & \multicolumn{1}{c}{$27.41$} &
      \multicolumn{1}{c}{$29.60$} & \multicolumn{1}{c}{$29.81$} & \multicolumn{1}{c}{\hlgreen{$\mathbf{32.54}$}} \\[+.95ex]
     & \multicolumn{1}{c}{SSIM$\uparrow$} & \multicolumn{1}{c}{$0.622$} & \multicolumn{1}{c}{$0.731$} & \multicolumn{1}{c}{$0.714$} &
      \multicolumn{1}{c}{$0.799$} & \multicolumn{1}{c}{$0.783$} & \multicolumn{1}{c}{\hlgreen{$\mathbf{0.881}$}} \\[+.75ex]
      & \multicolumn{1}{c}{LPIPS$\downarrow$} & \multicolumn{1}{c}{$0.333$} & \multicolumn{1}{c}{$0.235$} & \multicolumn{1}{c}{$0.256$} &
      \multicolumn{1}{c}{$0.147$} & \multicolumn{1}{c}{$0.174$}  & \multicolumn{1}{c}{\hlgreen{$\mathbf{0.119}$}} \\
      \bottomrule
    \end{tabular}
    }
    \label{table:comparison_table1}
\end{table*}

\begin{figure*}[t]
\begin{center}
\includegraphics[width=0.72\textwidth]{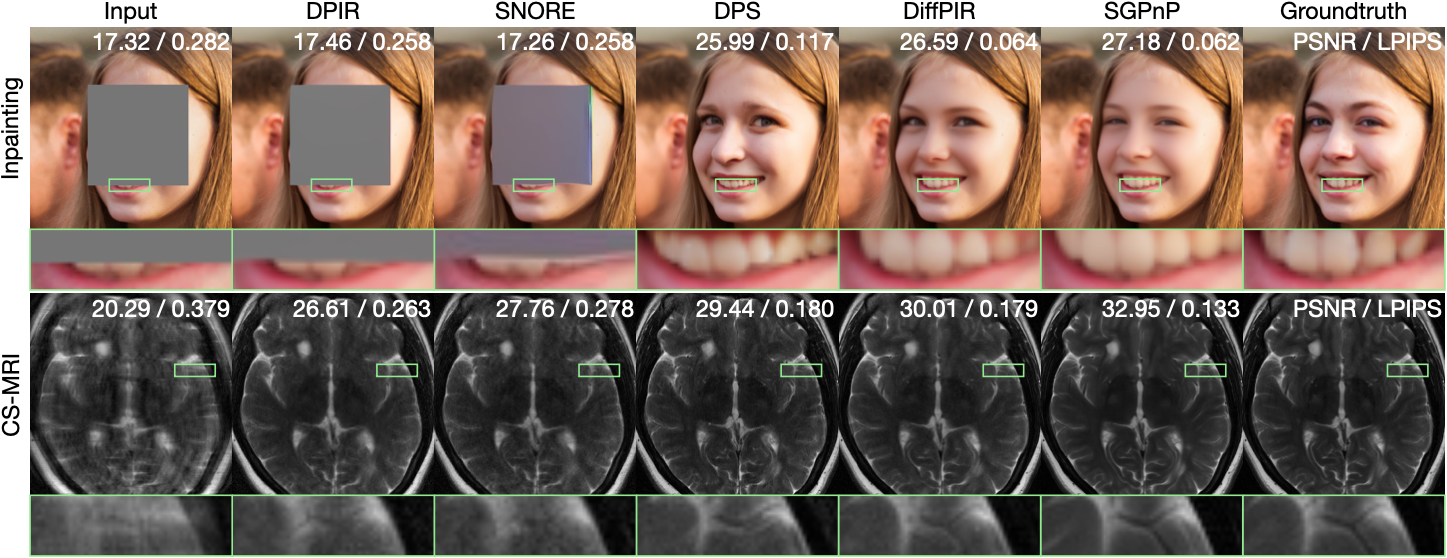}
\end{center}
\caption{
Qualitative visual comparison corresponding to the quantitative results in Table~\ref{table:comparison_table1}. Note that for challenging box inpainting, the deterministic (DPIR) and stochastic (SNORE) PnP methods, both using a denoiser trained on a low-noise range, produce incomplete reconstructions, whereas SGPnP produces more plausible image completions.}
\vspace{-.3cm}
\label{fig:visual_comparison1}
\end{figure*}

\begin{table}[t]
    \centering
    \caption{\small Quantitative evaluation of image deblurring comparing PnP methods using DRUNet with score-based variants obtained by replacing the denoiser with a diffusion model.
    \hlgreen{\textbf{Best values}} are color-coded for each metric.}
    \renewcommand{\arraystretch}{0.85}
    \resizebox{0.42\textwidth}{!}{
    \begin{tabular}{@{}p{3.4cm}p{0.5cm}p{0.5cm}p{0.5cm}@{}p{0.2cm}@{}p{0.5cm}p{0.5cm}p{0.5cm}@{}}
    \toprule
    & \multicolumn{1}{c}{\textbf{PSNR}$\uparrow$} & \multicolumn{1}{c}{\textbf{SSIM}$\uparrow$} & \multicolumn{1}{c}{\textbf{LPIPS}$\downarrow$} \\ 
    \cmidrule{1-4}\\[-1.9ex]
    Measurement  & \multicolumn{1}{c}{23.69} & \multicolumn{1}{c}{0.703} & \multicolumn{1}{c}{0.326} \\[+.1ex] \cdashline{1-4} \\[-1.6ex]
    DPIR & \multicolumn{1}{c}{33.42} & \multicolumn{1}{c}{0.923} & \multicolumn{1}{c}{0.123}  \\ [+.25ex]
    SDPnP-DPIR & \multicolumn{1}{c}{\hlbest{\textbf{34.14}}} & \multicolumn{1}{c}{\hlbest{\textbf{0.928}}} & \multicolumn{1}{c}{\hlbest{\textbf{0.114}}}  \\[+.2ex] \cdashline{1-4} \\[-1.3ex]
    PnP-ADMM & \multicolumn{1}{c}{33.31} & \multicolumn{1}{c}{0.935} & \multicolumn{1}{c}{0.102} \\[+.25ex]
    SDPnP-ADMM & \multicolumn{1}{c}{\hlbest{\textbf{34.06}}} & \multicolumn{1}{c}{\hlbest{\textbf{0.935}}} & \multicolumn{1}{c}{\hlbest{\textbf{0.096}}} \\[+.2ex] \cdashline{1-4} \\[-1.3ex]
    PGM & \multicolumn{1}{c}{33.29} & \multicolumn{1}{c}{0.930} & \multicolumn{1}{c}{0.090} \\[+.25ex]
    SDPnP-PGM & \multicolumn{1}{c}{\hlbest{\textbf{34.07}}} & \multicolumn{1}{c}{\hlbest{\textbf{0.935}}} & \multicolumn{1}{c}{\hlbest{\textbf{0.070}}} \\
    \bottomrule
    \end{tabular}
    }
    \label{table:comparison_table}
    \vspace{-0.5cm}
\end{table}

\vspace{-0.15cm}
\section{Numerical evaluation}
\label{sec:num_eval}
\vspace{-0.15cm}

We evaluate the proposed framework on two distinct modalities: RGB face images from Flickr-Faces-HQ (FFHQ)~\cite{karras2019ffhq} and complex-valued multi-coil MRI data from fastMRI~\cite{zbontar2018fastmri1, knoll2020fastmri2}.
Our evaluation is structured into three parts:
(1) \emph{Score Prior Adaptation}: We validate the effectiveness of replacing classical CNN denoisers with our score-adapted priors (Section~\ref{sec:score_adaptation}) within PnP frameworks (DPIR, PnP-ADMM, PGM).
(2) \emph{Stochastic Generative PnP (SGPnP)}: We compare our proposed stochastic framework (Section~\ref{sec:stochastic_pnp}) against traditional PnP algorithms, stochastic PnP, and diffusion-based solvers (see details in Section \ref{subsec:experiment_setup}).
(3) \emph{Ablation of Noise Injection}: We isolate the impact of the stochastic updates by comparing our method directly against its deterministic counterpart using the exactly \emph{same score prior}.

\subsection{Experimental Setup}
\label{subsec:experiment_setup}

\medskip\noindent
\textbf{Datasets \& Pretrained Models.} For FFHQ, we use the first 100 images ($256 \times 256$) as the test set and adapt the pretrained diffusion model from \cite{chung2023dps}.
Following the training convention of this model, RGB intensities are mapped linearly from $[0,1]$ to $[-1,1]$.
For fastMRI, we use 100 multi-coil brain scans (center-cropped to $256\times256$, and min--max normalized per slice so that its magnitude lies in $[0,1]$ with the phase preserved) and adapt the score model from \cite{park2024measurementdiffusion}.
To ensure fair comparison, we train the DRUNet~\cite{zhang2021dpir} baseline on both datasets, covering the noise level range $[0, 0.192]$ with noise-conditioning channels as per \cite{zhang2021dpir}.

\begin{table*}
    \centering
    \small
    \caption{
    {Quantitative comparison of each PnP algorithm without and with noise injection, using the same score-based prior, on FFHQ (inpainting, $4\times$ super-resolution, deblurring)
    and fastMRI ($4\times$ CS-MRI). ``w/o noise injection'' sets $\sigma^{\text{inject}}=0$ (SDPnP) and ``w/ noise injection'' uses $\sigma^{\text{inject}}>0$ (SGPnP).}
    \hlgreen{\textbf{Best values}} are highlighted for each metric, inverse problem, and PnP algorithm.
    With the same prior, noise injection improves nearly every entry.}
    \vspace{0.15cm}
    \renewcommand{\arraystretch}{0.7}
    \resizebox{0.95\textwidth}{!}{
    \begin{tabular}{@{}p{2.cm}p{0.1cm}p{0.5cm}p{0.5cm}p{0.5cm}@{}p{0.2cm}@{}p{0.5cm}p{0.5cm}p{0.5cm} p{0.5cm}p{0.5cm}@{}}
    \toprule
     &  &  & \multicolumn{2}{c}{DPIR} & \multicolumn{2}{c}{ADMM} & \multicolumn{2}{c}{RED} & \multicolumn{2}{c}{PGM} \\ \cmidrule(l{0.3em}r{0.3em}){4-5} \cmidrule(l{0.3em}r{0.3em}){6-7} \cmidrule(l{0.3em}r{0.3em}){8-9} \cmidrule(l{0.3em}r{0.3em}){10-11}
    \textbf{Testing data} &  & \multicolumn{1}{c}{\text{Input}} & \multicolumn{1}{c}{\shortstack{w/o noise\\injection}} & \multicolumn{1}{c}{\shortstack{w/ noise\\injection}} & \multicolumn{1}{c}{\shortstack{w/o noise\\injection}} & \multicolumn{1}{c}{\shortstack{w/ noise\\injection}} & \multicolumn{1}{c}{\shortstack{w/o noise\\injection}} & \multicolumn{1}{c}{\shortstack{w/ noise\\injection}} & \multicolumn{1}{c}{\shortstack{w/o noise\\injection}} & \multicolumn{1}{c}{\shortstack{w/ noise\\injection}} \\
    \cmidrule{1-11}\\ \noalign{\vskip -1.9ex}
        \multirow{3}{*}{\textbf{Inpainting}} & \multicolumn{1}{c}{PSNR$\uparrow$}  & \multicolumn{1}{c}{$18.17$} & \multicolumn{1}{c}{$18.29$} &
      \multicolumn{1}{c}{\hlgreen{$\mathbf{25.08}$}} & \multicolumn{1}{c}{$20.51$} & \multicolumn{1}{c}{\hlgreen{$\mathbf{24.12}$}}   & \multicolumn{1}{c}{$23.14$} & \multicolumn{1}{c}{\hlgreen{$\mathbf{24.62}$}}   & \multicolumn{1}{c}{$22.35$} & \multicolumn{1}{c}{\hlgreen{$\mathbf{25.21}$}} \\[+.95ex]
     & \multicolumn{1}{c}{SSIM$\uparrow$} & \multicolumn{1}{c}{$0.766$} & \multicolumn{1}{c}{$0.789$} &
      \multicolumn{1}{c}{\hlgreen{$\mathbf{0.869}$}} & \multicolumn{1}{c}{$0.811$} & \multicolumn{1}{c}{\hlgreen{$\mathbf{0.852}$}}   & \multicolumn{1}{c}{$0.830$} & \multicolumn{1}{c}{\hlgreen{$\mathbf{0.865}$}}   & \multicolumn{1}{c}{$0.847$} & \multicolumn{1}{c}{\hlgreen{$\mathbf{0.874}$}} \\[+.75ex]
      & \multicolumn{1}{c}{LPIPS$\downarrow$} & \multicolumn{1}{c}{$0.289$} & \multicolumn{1}{c}{$0.259$} &
      \multicolumn{1}{c}{\hlgreen{$\mathbf{0.109}$}} & \multicolumn{1}{c}{$0.227$} & \multicolumn{1}{c}{\hlgreen{$\mathbf{0.140}$}}   & \multicolumn{1}{c}{$0.150$} & \multicolumn{1}{c}{\hlgreen{$\mathbf{0.111}$}}   & \multicolumn{1}{c}{$0.141$} & \multicolumn{1}{c}{\hlgreen{$\mathbf{0.108}$}} \\[+.5ex]  
      \cdashline{1-11} \\ \noalign{\vskip -1.ex}
    \multirow{3}{*}{\textbf{Super-resolution}} & \multicolumn{1}{c}{PSNR$\uparrow$}  & \multicolumn{1}{c}{$24.61$} & \multicolumn{1}{c}{$27.22$} &
      \multicolumn{1}{c}{\hlgreen{$\mathbf{29.76}$}} & \multicolumn{1}{c}{$27.24$} & \multicolumn{1}{c}{\hlgreen{$\mathbf{29.60}$}}   & \multicolumn{1}{c}{$27.49$} & \multicolumn{1}{c}{\hlgreen{$\mathbf{29.38}$}}   & \multicolumn{1}{c}{$28.75$} & \multicolumn{1}{c}{\hlgreen{$\mathbf{29.90}$}} \\[+.95ex]
     & \multicolumn{1}{c}{SSIM$\uparrow$} & \multicolumn{1}{c}{$0.778$} & \multicolumn{1}{c}{$0.830$} &
      \multicolumn{1}{c}{\hlgreen{$\mathbf{0.871}$}} & \multicolumn{1}{c}{$0.849$} & \multicolumn{1}{c}{\hlgreen{$\mathbf{0.868}$}}   & \multicolumn{1}{c}{$0.815$} & \multicolumn{1}{c}{\hlgreen{$\mathbf{0.864}$}}   & \multicolumn{1}{c}{$0.854$} & \multicolumn{1}{c}{\hlgreen{$\mathbf{0.874}$}} \\[+.75ex]
      & \multicolumn{1}{c}{LPIPS$\downarrow$} & \multicolumn{1}{c}{$0.305$} & \multicolumn{1}{c}{$0.234$} &
      \multicolumn{1}{c}{\hlgreen{$\mathbf{0.194}$}} & \multicolumn{1}{c}{$0.209$} & \multicolumn{1}{c}{\hlgreen{$\mathbf{0.201}$}}   & \multicolumn{1}{c}{$0.221$} & \multicolumn{1}{c}{\hlgreen{$\mathbf{0.197}$}}   & \multicolumn{1}{c}{$0.202$} & \multicolumn{1}{c}{\hlgreen{$\mathbf{0.194}$}} \\[+.5ex]    
      \cdashline{1-11} \\ \noalign{\vskip -1.ex}
      \multirow{3}{*}{\textbf{Deblurring}} & \multicolumn{1}{c}{PSNR$\uparrow$}  & \multicolumn{1}{c}{$23.69$} & \multicolumn{1}{c}{$34.14$} &
      \multicolumn{1}{c}{\hlgreen{$\mathbf{34.28}$}} & \multicolumn{1}{c}{$34.06$} & \multicolumn{1}{c}{\hlgreen{$\mathbf{34.36}$}}   & \multicolumn{1}{c}{$30.70$} & \multicolumn{1}{c}{\hlgreen{$\mathbf{33.11}$}}   & \multicolumn{1}{c}{$34.07$} & \multicolumn{1}{c}{\hlgreen{$\mathbf{34.52}$}}  \\[+.95ex]
     & \multicolumn{1}{c}{SSIM$\uparrow$} & \multicolumn{1}{c}{$0.703$} & \multicolumn{1}{c}{$0.928$} &
      \multicolumn{1}{c}{\hlgreen{$\mathbf{0.935}$}} & \multicolumn{1}{c}{\hlgreen{$\mathbf{0.935}$}} & \multicolumn{1}{c}{$0.934$}   & \multicolumn{1}{c}{$0.886$} & \multicolumn{1}{c}{\hlgreen{$\mathbf{0.916}$}}   & \multicolumn{1}{c}{$0.935$} & \multicolumn{1}{c}{\hlgreen{$\mathbf{0.935}$}} \\[+.75ex]
      & \multicolumn{1}{c}{LPIPS$\downarrow$} & \multicolumn{1}{c}{$0.326$} & \multicolumn{1}{c}{$0.114$} &
      \multicolumn{1}{c}{\hlgreen{$\mathbf{0.095}$}} & \multicolumn{1}{c}{$0.096$} & \multicolumn{1}{c}{\hlgreen{$\mathbf{0.093}$}}   & \multicolumn{1}{c}{$0.105$} & \multicolumn{1}{c}{\hlgreen{$\mathbf{0.089}$}}   & \multicolumn{1}{c}{$0.070$} & \multicolumn{1}{c}{\hlgreen{$\mathbf{0.068}$}} \\
      \bottomrule \\ \noalign{\vskip -1.2ex}
      \multirow{3}{*}{\textbf{CS-MRI}} & \multicolumn{1}{c}{PSNR$\uparrow$}  & \multicolumn{1}{c}{$22.43$} & \multicolumn{1}{c}{$32.73$} &
      \multicolumn{1}{c}{\hlgreen{$\mathbf{32.91}$}} & \multicolumn{1}{c}{$33.19$} & \multicolumn{1}{c}{\hlgreen{$\mathbf{33.32}$}}   & \multicolumn{1}{c}{$32.28$} & \multicolumn{1}{c}{\hlgreen{$\mathbf{32.31}$}}   & \multicolumn{1}{c}{$32.34$} & \multicolumn{1}{c}{\hlgreen{$\mathbf{32.54}$}} \\[+.95ex]
     & \multicolumn{1}{c}{SSIM$\uparrow$} & \multicolumn{1}{c}{$0.622$} & \multicolumn{1}{c}{$0.870$} &
      \multicolumn{1}{c}{\hlgreen{$\mathbf{0.880}$}} & \multicolumn{1}{c}{$0.891$} & \multicolumn{1}{c}{\hlgreen{$\mathbf{0.895}$}}   & \multicolumn{1}{c}{$0.865$} & \multicolumn{1}{c}{\hlgreen{$\mathbf{0.866}$}}   & \multicolumn{1}{c}{\hlgreen{$\mathbf{0.882}$}} & \multicolumn{1}{c}{$0.881$} \\[+.75ex]
      & \multicolumn{1}{c}{LPIPS$\downarrow$} & \multicolumn{1}{c}{$0.333$} & \multicolumn{1}{c}{$0.153$} &
      \multicolumn{1}{c}{\hlgreen{$\mathbf{0.151}$}} & \multicolumn{1}{c}{$0.143$} & \multicolumn{1}{c}{\hlgreen{$\mathbf{0.111}$}}   & \multicolumn{1}{c}{$0.129$} & \multicolumn{1}{c}{\hlgreen{$\mathbf{0.128}$}}   & \multicolumn{1}{c}{$0.120$} & \multicolumn{1}{c}{\hlgreen{$\mathbf{0.119}$}} \\
    \bottomrule
    \end{tabular}
    }
    \label{table:comparison_table2}
\end{table*}

\begin{figure*}[t]
\begin{center}
\includegraphics[width=0.85\textwidth]{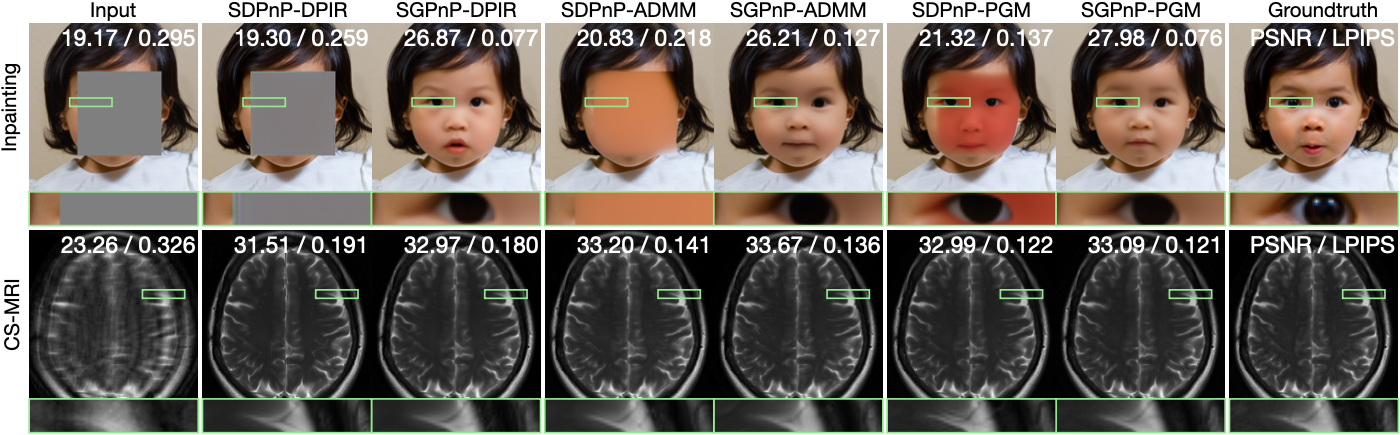}
\end{center}
\caption{Qualitative visual comparison corresponding to the quantitative results in Table~\ref{table:comparison_table2}, showing each PnP algorithm without noise injection (SDPnP) and with noise injection (SGPnP). Using the same score-based prior, stochastic noise injection produces more realistic completions in challenging box inpainting and improved reconstructions for accelerated MRI.}
\vspace{-.3cm}
\label{fig:visual_comparison2}
\end{figure*}

\medskip\noindent
\textbf{Baselines.} We compare against three categories of methods:
(1) \emph{Deterministic PnP}: PnP-ADMM~\cite{Venkatakrishnan.etal2013, Chan.etal2016}, DPIR~\cite{zhang2021dpir}, and PGM~\cite{combettes2011proximal} using the DRUNet prior.
(2) \emph{Stochastic PnP}: Stochastic denoising regularization (SNORE)~\cite{renaud2024snore} using the DRUNet prior.
(3) \emph{Diffusion Solvers}: Diffusion posterior sampling (DPS)~\cite{chung2023dps} and denoising diffusion models for plug-and-play image restoration (DiffPIR)~\cite{zhu2023DiffPIR}, which use the same diffusion priors as SGPnP.
All baselines are reproduced using their official repositories and the DeepInv library~\cite{tachella2023deepinverse} to ensure consistent forward operators and data-consistency steps.

\medskip\noindent
\textbf{Inverse Problems.} We evaluate performance on the following tasks. On FFHQ, we perform box inpainting ($128 \times 128$ center mask), motion deblurring, and $4\times$ super-resolution. On fastMRI, we perform accelerated reconstruction ($4\times$ acceleration).
In all experiments, measurements are synthetically generated from ground-truth images using the stated forward operators and corrupted by additive Gaussian noise with noise level of $0.02$.
We initialize with $\x_0=\vec{A}^{H}\y$. However, for tasks where $\y$ already lies in the image domain, such as inpainting and deblurring, we directly set $\x_0 = \y$.

\medskip\noindent
\textbf{Implementation Details.} We optimize hyperparameters for all methods via grid search on a held-out validation set of 50 samples, selecting parameters that maximize a trade-off between PSNR and LPIPS.
The final hyperparameters, including step sizes, noise annealing schedules, and numbers of iterations, are reported in Appendix~\ref{app:hyperparameters}.
Quantitative assessment relies on three standard metrics: peak signal-to-noise ratio (PSNR) and structural similarity index measure (SSIM) to measure fidelity and structural preservation, and learned perceptual image patch similarity (LPIPS)~\cite{zhang2018lpips} to quantify perceptual quality.

\subsection{Impact of Score-Based Priors in Deterministic PnP}

First, we assess the benefit of replacing the CNN-based DRUNet prior with the adapted score-based diffusion model (SBDM) prior (formulated in Section~\ref{sec:score_adaptation}) within PnP-ADMM, DPIR, and PGM.
Table~\ref{table:comparison_table} demonstrates that this substitution yields consistent improvements across all metrics.
These results highlight the benefit of the broad noise-level coverage provided by SBDMs: performance improves when the noise parameters are tuned beyond the low-noise range typically used with conventional PnP denoisers.

\subsection{Performance of Stochastic Generative PnP}

Next, we evaluate our stochastic generative PnP framework against baselines on two primary domains: natural image restoration (box inpainting, deblurring, and super-resolution) and accelerated MRI reconstruction.
Visual results in Figure~\ref{fig:visual_comparison1} highlight a critical distinction: while deterministic PnP and even stochastic PnP (SNORE), both with a denoiser trained on a low-noise range, fail to plausibly fill large missing regions in box inpainting, our method recovers structurally coherent content in the masked region.
Quantitatively (Table~\ref{table:comparison_table1}), our SGPnP-PGM framework consistently outperforms both the PnP baselines and the pure diffusion solvers (DPS, DiffPIR) in PSNR and SSIM.
Additional visual comparisons are provided in Appendix~\ref{app:more_visualization}.

\subsection{Stochastic vs. Deterministic Generative PnP}

We isolate the contribution of the proposed noise injection mechanism by comparing our stochastic generative PnP against a deterministic version of itself (i.e., using the same SBDM prior but with $\sigma^{\text{inject}}=0$ in \eqref{eq:stochastic_pnp_template}).
Figure~\ref{fig:visual_comparison2} reveals that even with a powerful score prior, the deterministic variant generally struggles to converge to realistic solutions in severely ill-posed tasks.
In contrast, SGPnP substantially improves reconstruction quality across all four PnP frameworks. Table~\ref{table:comparison_table2} confirms substantial quantitative gains from noise injection.

\section{Conclusion}

We bridge optimization-based PnP reconstruction and score-based diffusion priors, enabling pretrained score-based models to be used directly within PnP solvers.
We further introduce a stochastic generative PnP (SGPnP) framework that injects
noise before denoising while allowing the injected noise level to differ
from the denoiser conditioning level, reflecting the additional corruption introduced by repeated data-consistency updates.
Our analysis shows that this noise injection induces optimization on a
Gaussian-smoothed composite objective and leads to stochastic dynamics
that avoid strict saddle points.
Experiments on natural images and multi-coil MRI demonstrate consistent
improvements over deterministic and stochastic PnP baselines and show
that the proposed framework substantially improves reconstruction quality in severely ill-posed settings, including large-mask inpainting.
Beyond these specific contributions, we hope this work serves as groundwork for bridging classical optimization perspectives with modern generative models.

{\appendix}


\section{Detailed Proofs of Theoretical Results}
\label{app:theorydetail}

\begin{lemma}
\label{lem:zero_mean_noise}
Let $\mathcal{H}_k$ denote the history of the optimization process up to iteration $k$, containing all past iterates and noise realizations.
The stochastic noise sequence $\{\w_k\}_{k\ge 0}$ satisfies $\mathbb E_{\vec{n}_k}[\w_k\mid \mathcal{H}_k] = 0$. By the law of total expectation, $\mathbb E[\w_k] = 0$.
\end{lemma}

\begin{proof}
Because the history $\mathcal{H}_k$ completely determines the current state $\x_k$, and the injected noise $\vec{n}_k \sim \mathcal{N}(\vec{0}, \vec{I})$ is sampled independently of $\mathcal{H}_k$, we can take the conditional expectation of the noise definition $\w_k$ in \eqref{eq:w_k} with respect to $\vec{n}_k$:
\begin{align*}
\mathbb{E}_{\vec{n}_k}[\w_k \mid \mathcal{H}_k] 
&= \mathbb{E}_{\vec{n}_k} \big[ U_{\sigma_k}(\x_k, \vec{n}_k) \mid \mathcal{H}_k \big] - \mathbb{E}_{\vec{n}} \big[ U_{\sigma_k}(\x_k, \vec{n}) \big] \;.
\end{align*}
By Assumption~\ref{ass:unbiased}, the denoiser acts as an unbiased estimator of the score, meaning $\mathbb{E}_{\vec{n}}[U_{\sigma_k}(\x_k,\vec{n})] = \nabla h_\sigma(\x_k)$.
Applying this identity to the conditional expectation of the current step
\begin{equation*}
\mathbb{E}_{\vec{n}_k}[\w_k \mid \mathcal{H}_k] = \nabla h_\sigma(\x_k) - \nabla h_\sigma(\x_k) = \vec{0} \;.
\end{equation*}
ensuring the stochastic updates introduce no systematic bias.

Finally, applying the law of total expectation over the history $\mathcal{H}_k$ yields the unconditional mean: $\mathbb{E}[\w_k] = \mathbb{E}_{\mathcal{H}_k} \big[ \mathbb{E}_{\vec{n}_k}[\w_k \mid \mathcal{H}_k] \big] = \vec{0}.$
\end{proof}

\begin{lemma}
\label{lem:neighborhood}
Let Assumptions~\ref{ass:smoothness}--\ref{ass:denoiser_variance} hold, fix a
noise level $\sigma>0$, and run the SGPnP iteration~\eqref{eq:iter_final_clean}
for $I$ inner-loop iterations with constant step size $\gamma \le 1/L$. Then
\begin{equation}
\begin{split}
\min_{0\le i<I}\;\mathbb{E}\!\left[\|\nabla f_\sigma(\x^{(i)})\|^2\right]
\;\le\;{}
& \frac{2\big(f_\sigma(\x^{(0)})-\inf_{\x} f_\sigma(\x)\big)}{\gamma I} \\
& +\; L\,\gamma\, V_{\max}(\sigma).
\end{split}
\end{equation}
Consequently, as $I$ increases, the best-iterate gradient norm approaches the residual level $L\gamma V_{\max}(\sigma)$, a floor determined by the step size
$\gamma$ and the stochastic-gradient variance bound $V_{\max}(\sigma)$.
\end{lemma}
\begin{proof}
Fix $\sigma>0$ and a constant step size $\gamma\le 1/L$. By
Assumption~\ref{ass:unbiased} (equivalently $\mathbb{E}[\w^{(i)}\mid\mathcal{H}^{(i)}]=0$),
$\mathbb{E}_{\vec{n}}[U_\sigma(\x,\vec{n})]=\nabla h_\sigma(\x)$, so with $\w^{(i)}$
defined in~\eqref{eq:w_k} the iteration~\eqref{eq:iter_final_clean} can be written
as a stochastic gradient step on $f_\sigma=g+h_\sigma$:
\begin{equation}
\x^{(i+1)}=\x^{(i)}-\gamma\big(\nabla f_\sigma(\x^{(i)})+\w^{(i)}\big),
\end{equation}
where the zero-mean property is the preceding lemma and
$\mathbb{E}[\|\w^{(i)}\|^2\mid\mathcal{H}^{(i)}]=\mathrm{Var}_{\vec{n}}(U_\sigma(\x^{(i)},\vec{n}))
\le V_{\max}(\sigma)$ by Assumption~\ref{ass:denoiser_variance}.
Here $f_\sigma^\star := \inf_{\x} f_\sigma(\x) > -\infty$ by
Assumption~\ref{ass:smoothness}.
By the $L$-smoothness of $f_\sigma$ (Assumption~\ref{ass:smoothness}),
\begin{align*}
& f_\sigma(\x^{(i+1)})
\le f_\sigma(\x^{(i)})
-\gamma \nabla f_\sigma(\x^{(i)})^{\top}\left(\nabla f_\sigma(\x^{(i)})+\w^{(i)}\right) \\
& +\frac{L\gamma^2}{2}\big\|\nabla f_\sigma(\x^{(i)})+\w^{(i)}\big\|^2.
\end{align*}
Taking $\mathbb{E}[\cdot\mid\mathcal{H}^{(i)}]$ and using
$\mathbb{E}[\w^{(i)}\mid\mathcal{H}^{(i)}]=0$ together with the variance bound,
\begin{align*}
&
\mathbb{E}[f_\sigma(\x^{(i+1)})\mid\mathcal{H}^{(i)}] \\
&\le f_\sigma(\x^{(i)})
-\gamma\|\nabla f_\sigma(\x^{(i)})\|^2
+\frac{L\gamma^2}{2}\Big(\|\nabla f_\sigma(\x^{(i)})\|^2+V_{\max}(\sigma)\Big)\notag\\
&= f_\sigma(\x^{(i)})
-\gamma\Big(1-\tfrac{L\gamma}{2}\Big)\|\nabla f_\sigma(\x^{(i)})\|^2
+\frac{L\gamma^2}{2}V_{\max}(\sigma).
\end{align*}
Since $\gamma\le 1/L$ implies $1-\tfrac{L\gamma}{2}\ge\tfrac12$, we obtain
$\mathbb{E}[f_\sigma(\x^{(i+1)})\mid\mathcal{H}^{(i)}]
\le f_\sigma(\x^{(i)})-\tfrac{\gamma}{2}\|\nabla f_\sigma(\x^{(i)})\|^2
+\tfrac{L\gamma^2}{2}V_{\max}(\sigma)$.
Taking total expectation and summing over $i=0,\dots,I-1$, the intermediate
$f_\sigma(\x^{(i)})$ terms cancel as
\begin{equation}
\frac{\gamma}{2}\sum_{i=0}^{I-1}\mathbb{E}\|\nabla f_\sigma(\x^{(i)})\|^2
\le f_\sigma(\x^{(0)})-f_\sigma^\star
+\frac{IL\gamma^2}{2}V_{\max}(\sigma).
\end{equation}
Dividing by $I\gamma/2$ and bounding the minimum by the average yields
\begin{equation}
\min_{0\le i<I}\mathbb{E}\|\nabla f_\sigma(\x^{(i)})\|^2
\le\frac{2\big(f_\sigma(\x^{(0)})-f_\sigma^\star\big)}{\gamma I}
+L\gamma V_{\max}(\sigma).
\end{equation}
The first term decreases with $I$, so for any $\delta>0$ a finite $I$ makes
$\min_{0\le i<I}\mathbb{E}\|\nabla f_\sigma(\x^{(i)})\|^2 \le L\gamma V_{\max}(\sigma) + \delta$.
\end{proof}

\begin{definition}
\label{def:sosp}
A state $\x^\star$ is an $(\epsilon_g, \epsilon_h)$-second-order stationary point of the twice-differentiable objective $f_\sigma$ if its gradient norm is bounded by $\|\nabla f_\sigma(\x^\star)\| \le \epsilon_g$ and its Hessian satisfies $\nabla^2 f_\sigma(\x^\star) \succcurlyeq -\epsilon_h \vec{I}$, where $\epsilon_g, \epsilon_h >0$. 
\end{definition}

\noindent \textbf{Theorem \ref{thm:escape_saddle}.}
\emph{
Run the SGPnP iteration in~\eqref{eq:iter_final_clean} under Assumptions~\ref{ass:smoothness}--\ref{ass:denoiser_variance}. Then, for any $\delta \in (0, 1)$, there exists a stepsize schedule $\{\gamma_k\}_{k\geq0}$ and number of iterations $K$ such that, with probability at least $1-\delta$, the iterates avoid strict saddle points of $f_\sigma$.
}

\begin{proof}
We analyze the saddle-point avoidance properties of the SGPnP iteration by using the theoretical framework established in \cite{daneshmand2018escaping}.
Specifically, we assume the algorithm follows the step-size schedule detailed in \cite[Table 3]{daneshmand2018escaping}.
Escaping strict saddle points requires the stochastic gradient dynamics to (i) form a bounded zero-mean sequence, and (ii) satisfy the Correlated Negative Curvature (CNC) condition (i.e., there exists a constant $c > 0$ such that $\mathbb{E}_{\vec{n}_k}\big[ ( \vec{v}_{\x_k}^\top \widetilde{\nabla} f_\sigma(\x_k, \vec{n}_k) )^2 \big] \ge c$).
This condition means that the second moment of the stochastic gradient along the minimum eigenvector $\vec{v}_{\x_k}$ is uniformly bounded away from zero.
Intuitively, the CNC condition guarantees that the inherent optimization noise consistently provides a random drift along the steepest descent directions, preventing the algorithm from getting trapped on flat saddle points.

Condition (i) holds via Lemma~\ref{lem:zero_mean_noise} and Assumption~\ref{ass:denoiser_variance}.
To verify Condition (ii), we consider the full stochastic gradient $\widetilde{\nabla} f_\sigma(\x_k, \vec{n}_k) \coloneqq \nabla g(\x_k) + U_\sigma(\x_k, \vec{n}_k)$.
Expanding the second moment of its projection along the direction $\vec{v}_{\x_k}$, we have
\begin{align*}
&\mathbb{E}_{\vec{n}_k}\big[ ( \vec{v}_{\x_k}^\top \widetilde{\nabla} f_\sigma(\x_k, \vec{n}_k) )^2 \big] \nonumber \\
&\quad \geq \mathrm{Var}_{\vec{n}_k} \left( \vec{v}_{\x_k}^\top \nabla g(\x_k) + \frac{1}{\sigma^2} \vec{v}_{\x_k}^\top (\x_k - \mathsf{D}_\theta(\x_k+\sigma \vec{n}_k)) \right) \\
&\quad = \frac{1}{\sigma^4} \mathrm{Var}_{\vec{n}_k} \left(\vec{v}_{\x_k}^\top (\mathsf{D}_\theta(\x_k+\sigma \vec{n}_k)) \right) \\
&\quad \geq \frac{c}{\sigma^4} \;.
\end{align*}
where the first inequality follows from the second-moment decomposition $\mathbb{E}[\Y^2] = (\mathbb{E}[\Y])^2 + \mathrm{Var}(\Y) \ge \mathrm{Var}(\Y)$, the second equality holds because the deterministic components ($\nabla g(\x_k)$ and $\x_k$) carry zero variance over $\vec{n}_k$, and the final inequality follows from Assumption~\ref{ass:denoiser_variance}.
Therefore, even at a strict saddle point where the deterministic gradient vanishes ($\nabla f_\sigma(\x^\dagger) = 0$), the denoiser's inherent variance strictly satisfies the CNC condition, ensuring the iterates receive sufficient energy to successfully escape.

Because the SGPnP dynamics satisfy both the zero-mean and CNC conditions, the high-probability guarantees of \cite{daneshmand2018escaping} apply directly. Consequently, by configuring the step size $\{\gamma_k\}_{k\geq0}$ and total number of iterations $K$ according to their prescribed theoretical bounds, the algorithm is formally guaranteed to converge to an $(\epsilon, \sqrt{\rho}\epsilon)$-second-order stationary point with probability at least $1 - \delta$, where $\delta \in (0, 1)$ is the target failure tolerance. Adjusting the target probability near $1$ ($\delta \to 0$) requires a correspondingly larger number of iterations $K$ and strictly bounded step sizes as dictated by their analysis.

While this mathematical guarantee holds, we emphasize that this rigorous parameter configuration is established purely for theoretical worst-case bounds. In practice, our framework remains highly effective even when running the algorithm with much simpler, empirically tuned step sizes and number of iterations.

\end{proof}

\noindent \textbf{Theorem \ref{thm:annealed}.}
\emph{
Suppose Assumptions~\ref{ass:smoothness}--\ref{ass:anneal_conv} hold.
Consider a schedule $\{(\sigma_k,\gamma_k)\}_{k\ge 0}$ with
$\sigma_k\to 0$ and $\gamma_k\to 0$ as $k\to\infty$, 
satisfying $\gamma_k\le 1/L_k$ and $\gamma_k L_k V_{\max}(\sigma_k)\to 0$, where $L_k\coloneqq L(\sigma_k)$.
At each stage $k$, run the inner loop 
until the returned iterate $\hat\x_k$ satisfies
$\mathbb{E}\|\nabla f_{\sigma_k}(\hat\x_k)\|^2 \le 2\,L_k\gamma_k V_{\max}(\sigma_k)$,
which is attainable in a finite $I_k$.
Then
\begin{equation}
\mathbb{E}\!\left[\|\nabla f_0(\hat\x_k)\|^2\right]\;\longrightarrow\;0
\quad\text{as}\quad k\to\infty,
\end{equation}
where $f_0(\x)=g(\x)+h_0(\x)$ is the exact objective.}
\begin{proof}
By coercivity of $f_0$ (Assumption~\ref{ass:anneal_conv}) and the one-step
descent inequality in the proof of Lemma~\ref{lem:neighborhood}
(namely $\mathbb{E}[f_\sigma(\x^{(i+1)})\mid\mathcal{H}^{(i)}]\le f_\sigma(\x^{(i)})
-\tfrac{\gamma}{2}\|\nabla f_\sigma(\x^{(i)})\|^2+\tfrac{L\gamma^2}{2}V_{\max}(\sigma)$),
the iterates remain within a compact set $\mathcal{K}$; consequently the initial
optimality gap at each stage is uniformly bounded,
$\sup_k\big(f_{\sigma_k}(\x_k^{(0)})-f_{\sigma_k}^\star\big)\le\Delta<\infty$,
where stage $k$ is warm-started at $\hat\x_{k-1}$.
Applying Lemma~\ref{lem:neighborhood} at stage $k$ with step size
$\gamma_k\le 1/L_k$, the bound
\begin{equation}
\min_{0\le i<I_k}\mathbb{E}\|\nabla f_{\sigma_k}(\x_k^{(i)})\|^2
\le \frac{2\Delta}{\gamma_k I_k} + L_k\gamma_k V_{\max}(\sigma_k)
\end{equation}
holds. Choosing $I_k\ge 2\Delta\big/\big(L_k\gamma_k^2 V_{\max}(\sigma_k)\big)$
forces the first term below $L_k\gamma_k V_{\max}(\sigma_k)$, so the returned (arg-min)
iterate $\hat\x_k$ satisfies
\begin{equation}
\mathbb{E}\|\nabla f_{\sigma_k}(\hat\x_k)\|^2
\le 2L_k\,\gamma_k V_{\max}(\sigma_k),
\end{equation}
which is finite for each $k$. By the coupling,
$\gamma_k L_k V_{\max}(\sigma_k)\to0$, hence
$\mathbb{E}\|\nabla f_{\sigma_k}(\hat\x_k)\|^2\to0$.
To transfer this to the un-smoothed objective, note that
$f_\sigma-f_0=h_\sigma-h_0$, so
$\nabla f_{\sigma_k}(\x)-\nabla f_0(\x)=\nabla h_{\sigma_k}(\x)-\nabla h_0(\x)$.
Writing $\epsilon_k:=\sup_{\x\in\mathcal{K}}\|\nabla h_{\sigma_k}(\x)-\nabla h_0(\x)\|$,
the uniform convergence in Assumption~\ref{ass:anneal_conv} together with
$\sigma_k\to0$ gives $\epsilon_k\to0$, and the inequality
$\|a+b\|^2\le2\|a\|^2+2\|b\|^2$ yields
\begin{equation}
\mathbb{E}\|\nabla f_0(\hat\x_k)\|^2
\le 2\,\mathbb{E}\|\nabla f_{\sigma_k}(\hat\x_k)\|^2
+2\,\epsilon_k^2
\;\longrightarrow\; 0 .
\end{equation}
Finally, since $\mathbb{E}\|\nabla f_0(\hat\x_k)\|^2\to0$, there is a subsequence
along which $\nabla f_0(\hat\x_k)\to 0$ almost surely. 
As $\{\hat\x_k\}\subset\mathcal{K}$ is bounded, 
there exists a convergent subsequence, 
still denoted $\{\hat\x_k\}$, 
such that $\nabla f_0(\hat\x_k)\to 0$ almost surely and $\hat\x_k\to\x^\dagger$. 
Since $\nabla f_0$ is continuous
(Assumption~\ref{ass:smoothness} for $g$, Assumption~\ref{ass:anneal_conv} for
$h_0$), $\nabla f_0(\x^\dagger)=0$, i.e.\ $\x^\dagger\in S_0$.
\end{proof}

The following lemma shows that for convex data-fidelity $g$, Assumption~\ref{ass:denoiser_variance} is automatically satisfied at saddle points.
\begin{lemma}
\label{lem:saddle_point_escape}
Suppose Assumptions~\ref{ass:smoothness} and \ref{ass:unbiased} hold. 
Additionally, assume a convex data fidelity term $g$.
Then, for any saddle point $\saddlept$,
\begin{equation*}
    \mathrm{Var}_{\vec{n}}( \v_{\saddlept}^{\top} \Dsf_{\theta}(\saddlept + \sigma \vec{n}) ) \geq \sigma^2 \;.
\end{equation*}
\end{lemma}

\begin{proof}
Since $g$ is convex, $\nabla^2 g(\saddlept) \succeq 0$.
We have $h_\sigma(\saddlept) = - \mathbb{E}_{\vec{n} \sim \mathcal{N}(0, \vec{I})}[ \log p_{\sigma}(\saddlept + \sigma \vec{n})]$, giving
\begin{align}
    \nabla_{\saddlept}^2 h_{\sigma}(\saddlept) & = - \mathbb{E}_{\vec{n}}[ \nabla_{\saddlept} (\nabla_{\saddlept} \log p_{\sigma}(\saddlept + \sigma \vec{n})) ] \nonumber \\
    & = - \mathbb{E}_{\vec{n}} \left[ \nabla_{\saddlept} \left( \frac{1}{\sigma^2} \Dsf_{\theta} ( \saddlept + \sigma \vec{n} ) - \frac{1}{\sigma^2} (\saddlept + \sigma \vec{n}) \right) \right] \nonumber \\
     & = \frac{1}{\sigma^2} \vec{I} - \frac{1}{\sigma^4} \mathbb{E}_{\vec{n}} \big[ \text{Cov}\left(\x | \saddlept + \sigma \vec{n} \right)  \big]  \;. \nonumber  %
\end{align}
In the second equality, we use Tweedie's formula. In the last line, we use the Miyasawa identity:
\begin{equation}
    \text{Cov}(\x | \saddlept + \sigma \vec{n}) = \sigma^2 J_{\Dsf_{\theta}}(\saddlept + \sigma \vec{n}) \;. \label{eq: cov equals jacobian}
\end{equation}
Let $\overline{\Sigma}_{\saddlept} = \mathbb{E}_{\vec{n}} \big[ \text{Cov}\left(\x | \saddlept + \sigma \vec{n} \right)  \big]$.
Define $\v_{\saddlept}$ to be the eigenvector associated with the minimum eigenvalue of $\nabla^2 f_{\sigma} (\saddlept)$.
The minimum eigenvalue is equivalent to minimizing the Rayleigh quotient $\frac{\v^{\mathsf{T}} \nabla^2 f_{\sigma} (\saddlept) \v}{ \| \v \|_2 }$ restricted to unit vectors $\v$,
\begin{align*}
    \min_{\|\v \|_2 = 1} \left\{ \v^{\top} \nabla^2 g(\saddlept) \v + \frac{1}{\sigma^2} - \frac{1}{\sigma^4} \v^{\mathsf{T}} \overline{\Sigma}_{\saddlept} \v \right\}  \;.
\end{align*}
Since $\saddlept$ is a saddle point, the minimum must be non-positive.
Therefore, 
\begin{align}
    0 & \geq \v_{\saddlept}^{\top}\nabla^2 g(\saddlept)\v_{\saddlept} + \frac{1}{\sigma^2} - \frac{1}{\sigma^4} \v_{\saddlept}^{\mathsf{T}} \overline{\Sigma}_{\saddlept} \v_{\saddlept} \geq  \frac{1}{\sigma^2} - \frac{1}{\sigma^4} \v_{\saddlept}^{\mathsf{T}} \overline{\Sigma}_{\saddlept} \v_{\saddlept} \nonumber \\
        & \implies \qquad \v_{\saddlept}^{\mathsf{T}} \overline{\Sigma}_{\saddlept} \v_{\saddlept} \geq \sigma^2 \;. \label{eq: quadratic greater than sigma squared}
\end{align}
Next, let $z_1 = \v_{\saddlept}^{\top} \Dsf_{\theta}(\saddlept + \sigma \vec{n})$ and $z_2 = \v_{\saddlept}^{\top} \vec{n} $.
Notice these are both scalars with 
$\mathbb{E}[z_2]=0$ and $\text{Var}(z_2) = 1^2 \| \v_{\saddlept} \|_2^2 = 1$.
Definition of covariance gives
\begin{align}
    \text{Cov}(z_1, z_2) & = \mathbb{E}[z_1 z_2] -  \mathbb{E}[z_1] \mathbb{E}[z_2] \nonumber \\
        & = \mathbb{E}_{\vec{n}} \big[ ( \v_{\saddlept}^\top \Dsf_{\theta}(\saddlept + \sigma \vec{n}) ) ( \vec{n}^\top \v_{\saddlept} ) \big] - 0 \nonumber \\
        & = \sigma \, \v_{\saddlept}^{\mathsf{T}} \, \mathbb{E}_{\vec{n}} \big[  J_{\Dsf_{\theta}}(\saddlept + \sigma \vec{n}) \big] \v_{\saddlept} \;. \label{eq: cov z1 and z2}
\end{align}
In the last equality, we use Stein's Lemma.
Combining the Cauchy-Schwarz covariance inequality, Eq.~\eqref{eq: cov equals jacobian}, Eq.~\eqref{eq: quadratic greater than sigma squared}, and Eq.~\eqref{eq: cov z1 and z2} gives the desired result
\begin{align*}
    & \text{Var}(z_1) \text{Var}(z_2) \geq \text{Cov}(z_1, z_2)^2 \\
    & \quad \implies \quad \text{Var}_{\vec{n}}(  \v_{\saddlept}^\top \Dsf_{\theta}(\saddlept + \sigma \vec{n}) ) \geq \frac{1}{\sigma^2} \left(\v_{\saddlept}^{\mathsf{T}}  \overline{\Sigma}_{\saddlept} \v_{\saddlept} \right)^2 \geq \sigma^2 \;.
\end{align*}
\end{proof}

\section{Implementation Details}
\label{app:hyperparameters}

This section summarizes the hyperparameter settings used in all
experiments. Table~\ref{tab:pnp_params} reports parameters for
PnP-based methods (deterministic, stochastic, and DRUNet-based),
and Table~\ref{tab:diffusion_params} lists parameters for diffusion
posterior sampling approaches.
Here, $\sigma_0^{\text{cond}}$ denotes the initial conditioning noise level at the
first iteration. Following~\cite{zhu2023DiffPIR}, this conditioning noise level is annealed on a logarithmic schedule from $\sigma^{\text{cond}}_{0}$ to $\sigma^{\text{cond}}_{K-1}$. The injected noise level is likewise annealed on a logarithmic schedule from $\sigma^{\text{inject}}_{0}$ to $\sigma^{\text{inject}}_{K-1}$.
For every method in Table~\ref{tab:pnp_params} (SGPnP, SDPnP, and DPIR), the conditioning noise level ends at the measurement noise level, $\sigma^{\text{cond}}_{K-1}=0.02$. Only SGPnP injects noise: its injected level ends at $\sigma^{\text{inject}}_{K-1}=0.02$ on FFHQ and $\sigma^{\text{inject}}_{K-1}=0.001$ on CS-MRI.
The step sizes follow the same rules with and without noise injection. The prior step size $\tau$ of PGM and RED is constant; the proximal step of ADMM, DPIR, and PGM follows the schedule of~\cite{zhu2023DiffPIR}, $\gamma_k=\gamma\,(\sigma^{\text{cond}}_k/\sigma^{\text{cond}}_{K-1})^2$. Table~\ref{tab:pnp_params} reports $\gamma$ and $\tau$.

Hyperparameters for all methods were selected via grid search on a
validation set. Step sizes were searched over the range $[0.01, 5.0]$
using 40 uniformly spaced samples. The initial conditioning noise level
was varied from the maximum noise
level supported by the denoiser down to the measurement noise level. The number of iterations was selected
from $\{10, 20, 50, 100, 200\}$.
Note that for the deterministic DPIR implementation with DRUNet prior on fastMRI, we observed that a higher range of step size was necessary to perform well, so we extended the step-size grid up to 50.

\begin{table}[t]
\setstretch{1.2}
\centering
\scriptsize
\caption{Hyperparameters for PnP-based reconstruction methods.
SGPnP denotes the proposed stochastic generative plug-and-play framework with noise injection ($\sigma^{\text{inject}} > 0$).
SDPnP denotes the proposed score-based deterministic plug-and-play framework without noise injection ($\sigma^{\text{inject}} = 0$).
DPIR denotes a deterministic PnP method based on DRUNet denoisers.}
\begin{tabular}{lcccccc}
\toprule
\textbf{Problem} & \textbf{Method} & $\gamma$ & $\tau$ & $\sigma_0^{\text{cond}}$ & $\sigma_0^{\text{inject}}$ & Iter \\
\midrule

\textbf{Inpainting} & SGPnP-ADMM & 1.7 & -- & 15 & 15 & 200 \\
          & SGPnP-DPIR     & 1.5 & -- & 15  & 15 & 200 \\
          & SGPnP-PGM      & 0.22 & 0.4 & 20 & 20 & 200 \\
          & SDPnP-ADMM & 1.32 & -- & 5 & 0 & 50 \\
          & SDPnP-DPIR     & 2.0 & -- & 0.3 & 0 & 10 \\
          & SDPnP-PGM      & 2.0 & 0.4 & 20 & 0 & 200 \\
          & DPIR    & 1.627 & -- & 0.1921 & 0 & 50 \\

\cmidrule{1-7}

\textbf{Deblur} & SGPnP-ADMM & 1.398 & -- & 7.5 & 7.5 & 200 \\
        & SGPnP-DPIR     & 1.55 & -- & 25 & 25 & 200 \\
        & SGPnP-PGM      & 0.6286 & 0.3 & 20 & 20 & 200 \\
        & SDPnP-ADMM & 1.32 & -- & 5 & 0 & 50 \\
        & SDPnP-DPIR     & 2.0 & -- & 2 & 0 & 10 \\
        & SDPnP-PGM      & 0.66 & 0.4 & 2.0 & 0 & 10 \\
        & DPIR    & 1.627 & -- & 0.1921 & 0 & 10 \\

\cmidrule{1-7}

\textbf{SR} & SGPnP-ADMM & 1.45 & -- & 50 & 50 & 200 \\
    & SGPnP-DPIR     & 2.25 & -- & 15 & 15 & 200 \\
    & SGPnP-PGM      & 0.6286 & 0.3 & 20 & 20 & 200 \\
    & SDPnP-ADMM & 1.32 & -- & 0.6 & 0 & 20 \\
    & SDPnP-DPIR     & 0.608 & -- & 157 & 0 & 50 \\
    & SDPnP-PGM      & 1.48 & 0.95 & 20 & 0 & 100 \\
    & DPIR    & 1.627 & -- & 0.1921 & 0 & 50 \\

\midrule

\textbf{CS-MRI} & SGPnP-ADMM & 2.19 & -- & 1.0 & 0.01 & 200 \\
        & SGPnP-DPIR     & 3.5 & -- & 0.1 & 0.01 & 200 \\
        & SGPnP-PGM      & 1.7 & 0.75 & 1.0 & 0.01 & 200 \\
        & SDPnP-ADMM & 1.5 & -- & 50 & 0 & 200 \\
        & SDPnP-DPIR     & 2.5 & -- & 0.9 & 0 & 200 \\
        & SDPnP-PGM      & 1.4 & 0.63 & 7.5 & 0 & 200 \\
        & DPIR    & 35.0 & -- & 0.192 & 0 & 10 \\

\bottomrule
\end{tabular}
\label{tab:pnp_params}
\end{table}

\begin{table}[htbp]
\setstretch{1.2}
\centering
\scriptsize
\caption{Hyperparameters for diffusion-based inverse problem solvers.}
\begin{tabular}{lcccc}
\toprule
\textbf{Problem} & \textbf{Method} & $\lambda$ & $\gamma$ & Iter \\
\midrule

\textbf{Inpainting} & DiffPIR & 2.0 & 1.0 & 200 \\
          & DPS     & 0.4 & --  & 1000 \\
\textbf{Deblur}   & DiffPIR & 1.0 & 1.0 & 200 \\
          & DPS     & 10.0 & -- & 1000 \\
\textbf{SR}       & DiffPIR & 1.0 & 1.0 & 200 \\
          & DPS     & 3.0 & -- & 1000 \\

\midrule

\textbf{CS-MRI} & DiffPIR & 0.75 & 1.0 & 200 \\
        & DPS     & 10.0 & -- & 1000 \\

\bottomrule
\end{tabular}
\label{tab:diffusion_params}
\end{table}

\begin{table}[t]
    \centering
    \small
    \caption{
    Ablation on noise-level coupling for fastMRI ($4\times$ CS-MRI)
    using score-based PnP solvers.
    Results compare the matched setting
    ($\sigma^{\text{cond}}=\sigma^{\text{inject}}$)
    and the decoupled setting
    ($\sigma^{\text{cond}}\neq\sigma^{\text{inject}}$).
    \hlgreen{\textbf{Best values}} are highlighted for each metric.
    }
    \vspace{0.15cm}
    \renewcommand{\arraystretch}{0.7}
    \resizebox{\columnwidth}{!}{
    \begin{tabular}{@{}p{0.1cm}p{0.5cm}p{0.5cm}p{0.5cm}@{}p{0.2cm}@{}p{0.5cm} p{0.5cm}p{0.5cm}@{}}
    \toprule
     &  & \multicolumn{2}{c}{SGPnP-DPIR}  & \multicolumn{2}{c}{SGPnP-ADMM}  & \multicolumn{2}{c}{SGPnP-PGM} \\ \cmidrule(l{0.3em}r{0.3em}){3-4} \cmidrule(l{0.3em}r{0.3em}){5-6} \cmidrule(l{0.3em}r{0.3em}){7-8} \noalign{\vskip 0 ex}
    & \multicolumn{1}{c}{\text{Input}} & \multicolumn{1}{c}{Match} & \multicolumn{1}{c}{Decouple}  & \multicolumn{1}{c}{Match} & \multicolumn{1}{c}{Decouple}  & \multicolumn{1}{c}{Match} & \multicolumn{1}{c}{Decouple} \\
    \cmidrule{1-8}\\ \noalign{\vskip -1.9ex}
     \multicolumn{1}{c}{PSNR$\uparrow$}  & \multicolumn{1}{c}{$22.43$} & \multicolumn{1}{c}{$30.99$} &
      \multicolumn{1}{c}{\hlgreen{$\mathbf{32.91}$}} & \multicolumn{1}{c}{$31.44$} & \multicolumn{1}{c}{\hlgreen{$\mathbf{33.32}$}}   & \multicolumn{1}{c}{$31.54$} & \multicolumn{1}{c}{\hlgreen{$\mathbf{32.54}$}} \\[+.95ex]
      \multicolumn{1}{c}{SSIM$\uparrow$} & \multicolumn{1}{c}{$0.622$} & \multicolumn{1}{c}{$0.856$} &
      \multicolumn{1}{c}{\hlgreen{$\mathbf{0.880}$}} & \multicolumn{1}{c}{$0.854$} & \multicolumn{1}{c}{\hlgreen{$\mathbf{0.895}$}}   & \multicolumn{1}{c}{$0.823$} & \multicolumn{1}{c}{\hlgreen{$\mathbf{0.881}$}} \\[+.75ex]
      \multicolumn{1}{c}{LPIPS$\downarrow$} & \multicolumn{1}{c}{$0.333$} & \multicolumn{1}{c}{\hlgreen{$\mathbf{0.122}$}} &
      \multicolumn{1}{c}{$0.151$} & \multicolumn{1}{c}{$0.137$} & \multicolumn{1}{c}{\hlgreen{$\mathbf{0.111}$}}   & \multicolumn{1}{c}{$0.169$} & \multicolumn{1}{c}{\hlgreen{$\mathbf{0.119}$}} \\
    \bottomrule
    \end{tabular}
    }
    \label{table:ablation_noiseunmatch}
\end{table}

\section{Further Experimental Results}
\label{app:more_visualization}

\subsection{Impact of Noise-Level Decoupling}
\label{sub:impact_of_noise_level_decouping}

We further analyze the effect of decoupling the injected noise level from the denoiser conditioning noise level in stochastic generative PnP (SGPnP). As discussed in the main text, intermediate PnP iterates contain not only injected stochastic perturbations but also residual measurement noise and forward-operator–induced artifacts. Conditioning the denoiser solely on
the injected noise may therefore underestimate the effective corruption level of intermediate reconstructions.

The matched configuration ($\sigma^{\text{cond}}=\sigma^{\text{inject}}$),
which is used in stochastic PnP approach SNORE~\cite{renaud2024snore}, assumes that the injected noise fully characterizes the corruption level seen by the denoiser.
In contrast, our framework allows these two noise levels to differ so
that the conditioning noise can also account for measurement-induced
artifacts introduced by repeated data-consistency updates.
Table~\ref{table:ablation_noiseunmatch} shows that this decoupling consistently improves reconstruction stability.

\subsection{Selecting $\sigma^{\text{inject}}$}
\label{sub:selecting_ratio}

SGPnP has two noise parameters, $\sigma^{\text{inject}}$ and $\sigma^{\text{cond}}$.
They need not be tuned independently: performance is governed by their ratio $r=\sigma^{\text{cond}}/\sigma^{\text{inject}}$, which reduces the search to a single parameter.
We sweep $r\in\{0.05,1,10,100,500\}$ on a validation subset of each of the four problems (Table~\ref{tab:sigma_inject_sweep}).
Two observations give a practical rule.
First, the response is flat over a wide band and fails only for $r\ll 1$.
At $r=0.05$ every problem collapses to $10$--$14$ dB, because over-injection drowns the measurement signal.
For $r\ge 1$ the curves are broad, and on deblurring and MRI they vary by under $0.1$ dB across $r\in[10,500]$.
Second, the peak location follows the problem geometry.
Inpainting and $4\times$ SR leave a large unmeasured subspace and need stronger stochastic exploration to fill it; they peak at $r\approx 1$.
Deblurring and $4\times$ CS-MRI are measurement-rich and dominated by data consistency; they peak on a flat plateau at $r\in[10,100]$.

For a new problem, choosing $r$ therefore reduces to comparing the default
$r=1$ ($\sigma^{\text{inject}}=\sigma^{\text{cond}}$) with one value on the plateau, e.g.,
$r=10$, rather than a fine search. Decoupling gains $0.23$~dB on deblurring and $1.00$~dB on
CS-MRI, whereas on inpainting and SR it lowers PSNR by $3$--$9$~dB. Measurement-rich problems
are therefore where the larger value pays off, and $r=1$ remains the safer starting point
when the problem type is unknown.

\begin{table}[htbp]
    \centering
    \small
    \renewcommand{\arraystretch}{1.1}
    \caption{
    Sensitivity of SGPnP to the noise-injection level. The control variable is the ratio $r=\sigma_\mathrm{cond}/\sigma_\mathrm{inject}$. \textbf{Best values} of PSNR are highlighted per row.
    Evaluation shows that underdetermined problems (inpainting, SR) peak at $r\approx 1$, while deblurring and MRI peak on a flat plateau at $r\in[10,100]$. The default $r=1$ ($\sigma_\mathrm{inject}=\sigma_\mathrm{cond}$) is optimal or near-optimal on all four problems.
    }
    \label{tab:sigma_inject_sweep}
    \resizebox{0.45\textwidth}{!}{
    \begin{tabular}{@{}lccccc@{}}
    \toprule
    Problem & $r{=}0.05$ & $r{=}1$ & $r{=}10$ & $r{=}100$ & $r{=}500$ \\
    \midrule
    Inpainting (box)             & $10.81$ & \hlgreen{$\mathbf{25.04}$} & $21.83$ & $21.39$ & $21.40$ \\
    Super-resolution ($4\times$) & $10.37$ & \hlgreen{$\mathbf{29.84}$} & $21.05$ & $21.61$ & $21.54$ \\
    Deblurring                   & $12.43$ & $33.93$ & \hlgreen{$\mathbf{34.16}$} & $34.09$ & $34.09$ \\
    CS-MRI ($4\times$)           & $13.74$ & $31.54$ & $32.53$ & \hlgreen{$\mathbf{32.54}$} & \hlgreen{$\mathbf{32.54}$} \\
    \bottomrule
    \end{tabular}
    }
\end{table}

\subsection{Impact of Noise-Level Coverage}
\label{sub:impact_of_noise_range}

We further investigate how the noise-level coverage of the denoiser
affects stochastic generative PnP (SGPnP) reconstruction.
Specifically, we compare SGPnP-PGM operating over the full noise range used during score-based model training with a variant that employs the same score-based denoiser but restricts it to a low-noise regime, matching the range typically used by baseline PnP methods, including deterministic DPIR~\cite{zhang2021dpir} and stochastic SNORE~\cite{renaud2024snore}.

Table~\ref{table:ablation_noiserange} shows that restricting SGPnP-PGM to the low-noise regime yields no meaningful improvement over deterministic DPIR and stochastic SNORE in the box inpainting problem.
In contrast, allowing SGPnP-PGM to fully leverage the wider noise range leads to a substantial improvement across all reconstruction metrics. These results suggest that broad noise-level coverage is a key factor enabling effective stochastic generative reconstruction.

\subsection{Additional Visual Comparisons}
\label{sub:additional_visual_comparisons}

We include additional visual comparisons to further illustrate the behavior of the proposed stochastic generative PnP method. 
Figure~\ref{fig:ablation_multi_ours} shows repeated reconstructions from the same measurements using the proposed method, demonstrating that the injected stochasticity still leads to visually realistic solutions across runs.
Table~\ref{table:comparison_single} reports the mean $\pm$ standard deviation across 10 independent runs of box inpainting over the test set, showing that SGPnP's trial-to-trial variation is small.
Figure~\ref{fig:ablation_multi_inpainting} presents additional box inpainting examples on more measurements, where DPIR and SNORE often produce incomplete reconstructions and deterministic PGM with score prior improves the result but still struggles in challenging cases; in contrast, the proposed method yields more plausible image completions.

\begin{table}[htbp]
    \centering
    \small
    \caption{
    Ablation study on noise-level coverage for box inpainting. We compare deterministic PnP (DPIR), a stochastic PnP baseline (SNORE), and SGPnP-PGM restricted to the same low-noise regime used by classical
    denoisers, and SGPnP-PGM leveraging the full noise range available during score-based model training.
    \hlgreen{\textbf{Best values}} are highlighted for each metric.
    }
    
    \vspace{0.15cm}
    \renewcommand{\arraystretch}{0.7}
    \resizebox{\columnwidth}{!}{
    \begin{tabular}{@{}p{0.1cm}p{0.5cm}p{0.5cm}p{0.5cm}@{}p{0.2cm}@{}p{0.5cm}@{}}
    \toprule
    & \multicolumn{1}{c}{\text{Input}} & \multicolumn{1}{c}{\text{DPIR} ($\sigma_{\text{low}}$)} & \multicolumn{1}{c}{\text{SNORE} ($\sigma_{\text{low}}$)}   & \multicolumn{1}{c}{\text{SGPnP-PGM ($\sigma_{\text{low}}$)}}  &  \multicolumn{1}{c}{\text{SGPnP-PGM ($\sigma_{\text{wide}}$)}} \\
    \cmidrule{1-6}\\ \noalign{\vskip -1.9ex}
    \multicolumn{1}{c}{PSNR$\uparrow$}  & \multicolumn{1}{c}{$18.17$} & \multicolumn{1}{c}{$18.42$} & \multicolumn{1}{c}{$18.50$}  & \multicolumn{1}{c}{$18.27$} & \multicolumn{1}{c}{\hlgreen{$\mathbf{25.21}$}} \\[+.95ex]
     \multicolumn{1}{c}{SSIM$\uparrow$} & \multicolumn{1}{c}{$0.766$} & \multicolumn{1}{c}{$0.797$} & \multicolumn{1}{c}{$0.799$}  & \multicolumn{1}{c}{$0.797$} & \multicolumn{1}{c}{\hlgreen{$\mathbf{0.874}$}} \\[+.75ex]
      \multicolumn{1}{c}{LPIPS$\downarrow$} & \multicolumn{1}{c}{$0.289$} & \multicolumn{1}{c}{$0.264$}  & \multicolumn{1}{c}{$0.250$} & \multicolumn{1}{c}{$0.257$}  & \multicolumn{1}{c}{\hlgreen{$\mathbf{0.108}$}} \\
      \bottomrule
    \end{tabular}
    }
    \label{table:ablation_noiserange}
\end{table}

\subsection{Runtime and Computational Cost}

\label{sub:runtime}

We report wall-clock runtime, network type, and iteration count $K$ for all solvers in Table~\ref{table:runtime}.
Two points follow.
First, noise injection adds no measurable cost.
SDPnP-RED (no injection) and SGPnP-RED (with injection) run at $2{,}300$ ms and $2{,}302$ ms per image, both at $K=200$.
Second, the gains of SGPnP do not come from extra denoiser evaluations.
SGPnP and DiffPIR both use $K=200$, while DPS uses $K=1000$.
In wall-clock time, SGPnP matches DiffPIR ($2{,}300$ ms) and is far cheaper than SNORE ($8{,}475$ ms) and DPS ($24{,}050$ ms).

\begin{table}[htbp]
    \centering
    \small
    \caption{
    Runtime comparison across PnP and diffusion-based solvers.
    $K$ denotes the number of iterations, and time/image reports the total wall-clock time per reconstruction.
    }
    \renewcommand{\arraystretch}{0.7}
    \resizebox{0.4\textwidth}{!}{
    \begin{tabular}{@{}llrr@{}}
    \toprule
    \textbf{Method} & \textbf{Network} & \textbf{$K$} & \textbf{time/image} \\
    \cmidrule{1-4}\\ \noalign{\vskip -1.9ex}
    DPIR (classical) & DRUNet    & $200$  & $1{,}130$ ms  \\[+.6ex]
    SNORE            & DRUNet    & $1500$ & $8{,}475$ ms  \\[+.6ex]
    DPS              & Diffusion & $1000$ & $24{,}050$ ms \\[+.6ex]
    DiffPIR          & Diffusion & $200$  & $2{,}300$ ms           \\[+.6ex]
    SDPnP-RED        & Diffusion & $200$  & $2{,}300$ ms  \\[+.6ex]
    SGPnP-RED        & Diffusion & $200$  & $2{,}302$ ms  \\[+.6ex]
    \bottomrule
    \end{tabular}
    }
    \label{table:runtime}
\end{table}

\begin{table}[b]
    \centering
    \small
    \caption{
    Comparison of SGPnP-RED under two schedules on FFHQ (inpainting, $4\times$ super-resolution, deblurring) and fastMRI ($4\times$ CS-MRI): the configuration used in the main experiments ($1$ inner step per noise level) and a staged configuration ($10$ inner steps per noise level). The two behave comparably across all tasks.
    }
    
    \vspace{0.15cm}
    \renewcommand{\arraystretch}{0.7}
    \resizebox{0.49\textwidth}{!}{
    \begin{tabular}{@{}p{2.8cm}p{0.1cm}p{1.cm}p{1.4cm}p{1.8cm}@{}}
    \toprule
    \textbf{Testing data} &  & \multicolumn{1}{c}{\text{Input}} & \multicolumn{1}{c}{\shortstack{SGPnP-RED\\(1 inner step,\\200 noise levels)}} & \multicolumn{1}{c}{\shortstack{SGPnP-RED\\(10 inner steps,\\200 noise levels)}} \\
    \cmidrule{1-5}\\ \noalign{\vskip -1.9ex}
        \multirow{3}{*}{\textbf{Inpainting}} & \multicolumn{1}{c}{PSNR$\uparrow$} & \multicolumn{1}{c}{$18.17$} & \multicolumn{1}{c}{$24.62$} & \multicolumn{1}{c}{$25.26$} \\[+.95ex]
     & \multicolumn{1}{c}{SSIM$\uparrow$} & \multicolumn{1}{c}{$0.766$} & \multicolumn{1}{c}{$0.865$} & \multicolumn{1}{c}{$0.867$} \\[+.75ex]
      & \multicolumn{1}{c}{LPIPS$\downarrow$} & \multicolumn{1}{c}{$0.289$} & \multicolumn{1}{c}{$0.111$} & \multicolumn{1}{c}{$0.120$} \\[+.5ex]
      \cdashline{1-5} \\ \noalign{\vskip -1.ex}
    \multirow{3}{*}{\textbf{Super-resolution}} & \multicolumn{1}{c}{PSNR$\uparrow$} & \multicolumn{1}{c}{$24.61$} & \multicolumn{1}{c}{$29.38$} & \multicolumn{1}{c}{$29.45$} \\[+.95ex]
     & \multicolumn{1}{c}{SSIM$\uparrow$} & \multicolumn{1}{c}{$0.778$} & \multicolumn{1}{c}{$0.864$} & \multicolumn{1}{c}{$0.866$} \\[+.75ex]
      & \multicolumn{1}{c}{LPIPS$\downarrow$} & \multicolumn{1}{c}{$0.305$} & \multicolumn{1}{c}{$0.197$} & \multicolumn{1}{c}{$0.199$} \\[+.5ex]
      \cdashline{1-5} \\ \noalign{\vskip -1.ex}
      \multirow{3}{*}{\textbf{Deblurring}} & \multicolumn{1}{c}{PSNR$\uparrow$} & \multicolumn{1}{c}{$23.69$} & \multicolumn{1}{c}{$33.11$} & \multicolumn{1}{c}{$34.07$} \\[+.95ex]
     & \multicolumn{1}{c}{SSIM$\uparrow$} & \multicolumn{1}{c}{$0.703$} & \multicolumn{1}{c}{$0.916$} & \multicolumn{1}{c}{$0.930$} \\[+.75ex]
      & \multicolumn{1}{c}{LPIPS$\downarrow$} & \multicolumn{1}{c}{$0.326$} & \multicolumn{1}{c}{$0.089$} & \multicolumn{1}{c}{$0.082$} \\
      \bottomrule \\ \noalign{\vskip -1.ex}
    \multirow{3}{*}{\textbf{CS-MRI}} & \multicolumn{1}{c}{PSNR$\uparrow$} & \multicolumn{1}{c}{$22.43$} & \multicolumn{1}{c}{$32.31$} & \multicolumn{1}{c}{$32.55$} \\[+.95ex]
     & \multicolumn{1}{c}{SSIM$\uparrow$} & \multicolumn{1}{c}{$0.622$} & \multicolumn{1}{c}{$0.866$} & \multicolumn{1}{c}{$0.870$} \\[+.75ex]
      & \multicolumn{1}{c}{LPIPS$\downarrow$} & \multicolumn{1}{c}{$0.333$} & \multicolumn{1}{c}{$0.128$} & \multicolumn{1}{c}{$0.111$} \\
      \bottomrule
    \end{tabular}
    }
    \label{table:theory_setting}
\end{table}

\subsection{Staged vs.\ Single-Step Schedules}

\label{sub:theory_matching}

The convergence analysis in Section~\ref{sec:theory_escape} assumes a staged
schedule that takes several steps per noise level before annealing, whereas the
SGPnP configuration used in the main experiments takes a single step per level. To
check that the analyzed and practical algorithms behave consistently, we run
SGPnP-RED with $10$ inner steps at each of $200$ noise levels and compare it
against the single-step configuration used in the main experiments.
Table~\ref{table:theory_setting} reports both on FFHQ (inpainting, $4\times$
super-resolution, deblurring) and fastMRI ($4\times$ CS-MRI). The two
configurations reach comparable reconstruction quality across all four tasks, with
PSNR differing by at most about $1$\,dB and SSIM by at most $0.014$, indicating
that the staged schedule analyzed by the theory and the single-step schedule run
in practice behave consistently.

\begin{figure}[t]
\begin{center}
\includegraphics[width=0.49\textwidth]{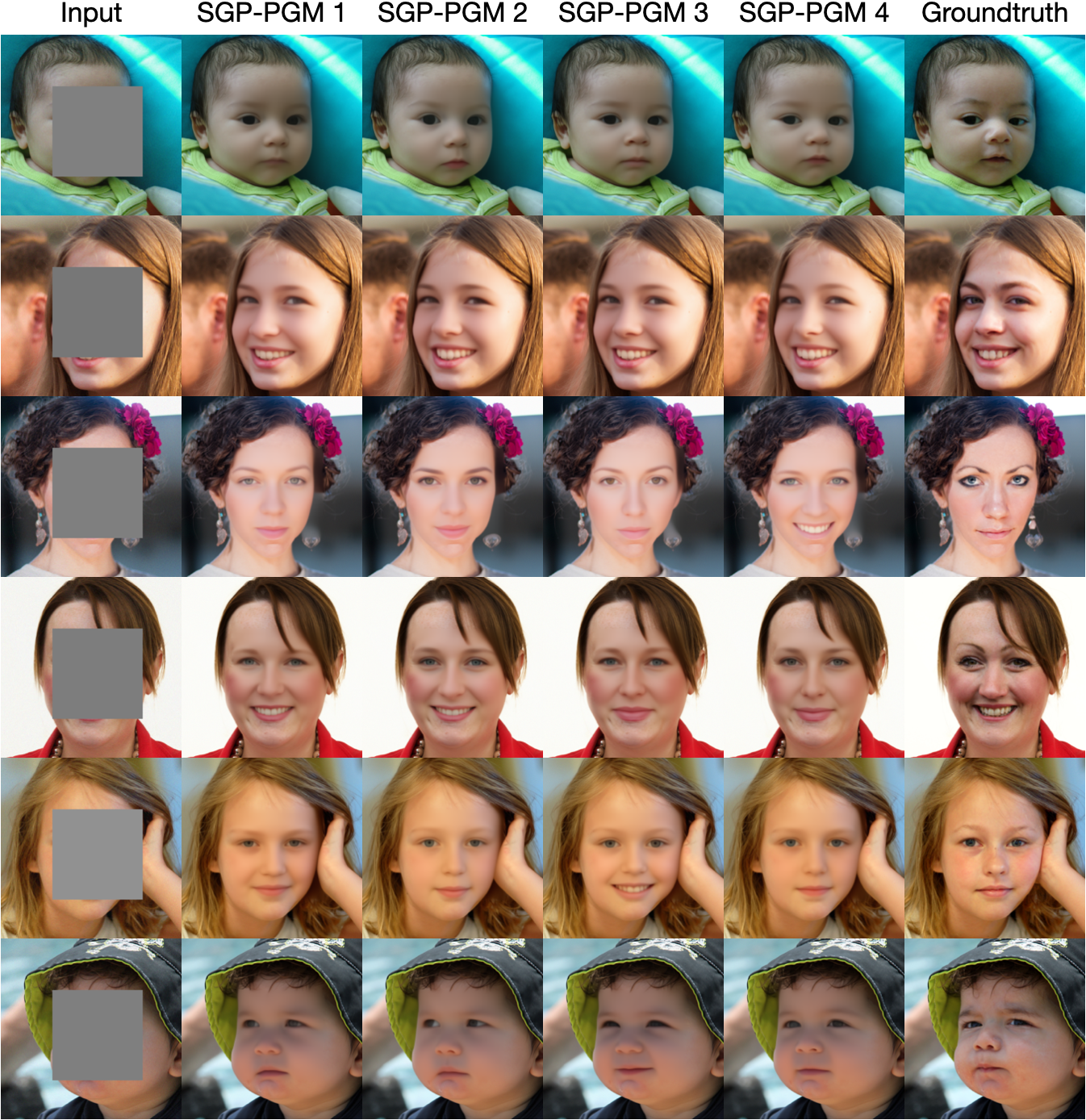}
\end{center}
\caption{
Box inpainting results obtained with SGPnP-PGM from repeated runs on the same measurements.
The method produces consistent reconstructions across runs.
}
\vspace{-.3cm}
\label{fig:ablation_multi_ours}
\end{figure}

\begin{table}[t]
    \centering
    \small
    \caption{
    Comparison between DPS, DiffPIR, and the two SGPnP variants (SGPnP-PGM and SGPnP-RED) in solving box inpainting.
    Each entry reports the mean $\pm$ standard deviation over 10 independent runs, with each run's metrics averaged over the test set.}
    \vspace{0.15cm}
    \renewcommand{\arraystretch}{0.7}
    \resizebox{0.49\textwidth}{!}{
    \begin{tabular}{@{}ccccc@{}}
    \toprule
     & \multicolumn{1}{c}{\text{DPS}} & \multicolumn{1}{c}{\text{DiffPIR}} & \multicolumn{1}{c}{\text{SGPnP-PGM}} & \multicolumn{1}{c}{\text{SGPnP-RED}} \\
    \cmidrule{1-5}\\ \noalign{\vskip -1.9ex}
    \multicolumn{1}{c}{PSNR$\uparrow$}  & $23.86 \pm 0.970$ & $25.00 \pm 0.696$ & $25.11 \pm 0.566$ & $24.88 \pm 0.527$ \\[+.75ex]
    \multicolumn{1}{c}{SSIM$\uparrow$}  & $0.835 \pm 0.0101$ & $0.859 \pm 0.0076$ & $0.873 \pm 0.0060$ & $0.852 \pm 0.0058$ \\[+.75ex]
    \multicolumn{1}{c}{LPIPS$\downarrow$}  & $0.152 \pm 0.0082$ & $0.113 \pm 0.0057$ & $0.108 \pm 0.0038$ & $0.131 \pm 0.0038$ \\
    \bottomrule
    \end{tabular}
    }
    \label{table:comparison_single}
\end{table}

\begin{table}[b]
    \centering
    \small
    \caption{
    Quantitative comparison between DPS and SGPnP-RED on the nonlinear deblurring problem (using the neural network degradation forward model from~\cite{chung2023dps}). Note that SGPnP can effectively address the nonlinear inverse problem.
    }
    \vspace{0.15cm}
    \renewcommand{\arraystretch}{0.7}
    \resizebox{0.38\textwidth}{!}{
    \begin{tabular}{@{}p{2.cm}p{0.1cm}p{0.5cm}@{}p{0.2cm}@{}p{0.5cm}@{}}
    \toprule
    \textbf{} &  & \multicolumn{1}{c}{\text{DPS}} & & \multicolumn{1}{c}{\text{SGPnP-RED}} \\
    \cmidrule{1-5}\\ \noalign{\vskip -1.9ex}
        \multicolumn{1}{l}{PSNR$\uparrow$}  & & \multicolumn{1}{c}{$25.40$} & & \multicolumn{1}{c}{$26.29$} \\[+.95ex]
        \multicolumn{1}{l}{SSIM$\uparrow$}  & & \multicolumn{1}{c}{$0.742$} & & \multicolumn{1}{c}{$0.772$} \\[+.75ex]
        \multicolumn{1}{l}{LPIPS$\downarrow$} & & \multicolumn{1}{c}{$0.220$} & & \multicolumn{1}{c}{$0.269$} \\
    \bottomrule
    \end{tabular}
    }
    \label{table:comparison_k800}
\end{table}

\begin{table}[htbp]
    \centering
    \small
    \caption{
    Denoising PSNR of the pretrained diffusion model conditioned on fractional time steps $t^*$ between trained entries, on $100$ held-out test images. PSNR decreases monotonically with the noise level $\sigma$, even with non-integer $t^*$.}
    \label{tab:interp_fractional}
    \vspace{0.15cm}
    \renewcommand{\arraystretch}{1.05}
    \begin{tabular}{@{}rrr@{}}
    \toprule
    \textbf{Time step $t^*$} & \textbf{Noise level $\sigma$} & \textbf{PSNR (dB)} \\
    \midrule
    2.00   & 0.0190  & 45.64 \\
    2.25   & 0.0199  & 45.37 \\
    2.50   & 0.0209  & 45.11 \\
    2.75   & 0.0218  & 44.87 \\
    3.00   & 0.0228  & 44.64 \\
    \cmidrule{1-3}
    30.00  & 0.1115  & 36.86 \\
    30.25  & 0.1124  & 36.83 \\
    30.50  & 0.1132  & 36.79 \\
    30.75  & 0.1140  & 36.76 \\
    31.00  & 0.1148  & 36.73 \\
    \cmidrule{1-3}
    100.00 & 0.3423  & 32.06 \\
    100.25 & 0.3431  & 32.05 \\
    100.50 & 0.3440  & 32.04 \\
    100.75 & 0.3448  & 32.03 \\
    101.00 & 0.3457  & 32.02 \\
    \cmidrule{1-3}
    768.00 & 19.9605 & 16.50 \\
    768.25 & 19.9995 & 16.50 \\
    768.50 & 20.0385 & 16.49 \\
    768.75 & 20.0776 & 16.48 \\
    769.00 & 20.1166 & 16.48 \\
    \bottomrule
    \end{tabular}
\end{table}

\begin{table}[t]
    \centering
    \small
    \caption{
    Comparison between PnP-Flow~\cite{martin2025pnpflow} and SGPnP-PGM on FFHQ (inpainting, $4\times$ super-resolution, deblurring). SGPnP-PGM results are those of Table~\ref{table:comparison_table1}.
    \hlgreen{\textbf{Best values}} are highlighted for each metric and inverse problem.
    }
    \vspace{0.15cm}
    \renewcommand{\arraystretch}{1.1}
    \resizebox{0.45\textwidth}{!}{
    \begin{tabular}{@{}llccc@{}}
    \toprule
    \textbf{Testing data} & & Input & PnP-Flow & SGPnP-PGM \\
    \midrule
    \multirow{3}{*}{\textbf{Inpainting}}
     & PSNR$\uparrow$    & $18.17$ & $24.02$ & \hlgreen{$\mathbf{25.21}$} \\
     & SSIM$\uparrow$    & $0.766$ & $0.845$ & \hlgreen{$\mathbf{0.874}$} \\
     & LPIPS$\downarrow$ & $0.289$ & $0.135$ & \hlgreen{$\mathbf{0.108}$} \\
    \cdashline{1-5}
    \multirow{3}{*}{\textbf{Super-resolution}}
     & PSNR$\uparrow$    & $24.61$ & $28.30$ & \hlgreen{$\mathbf{29.90}$} \\
     & SSIM$\uparrow$    & $0.778$ & $0.853$ & \hlgreen{$\mathbf{0.874}$} \\
     & LPIPS$\downarrow$ & $0.305$ & \hlgreen{$\mathbf{0.193}$} & $0.194$ \\
    \cdashline{1-5}
    \multirow{3}{*}{\textbf{Deblurring}}
     & PSNR$\uparrow$    & $23.69$ & $30.64$ & \hlgreen{$\mathbf{34.52}$} \\
     & SSIM$\uparrow$    & $0.703$ & $0.898$ & \hlgreen{$\mathbf{0.935}$} \\
     & LPIPS$\downarrow$ & $0.326$ & $0.110$ & \hlgreen{$\mathbf{0.068}$} \\
    \bottomrule
    \end{tabular}
    }
    \label{table:comparison_pnpflow}
\end{table}

\subsection{Nonlinear Inverse Problem}
\label{sub:nonlinear}

The main experiments use linear forward operators with Gaussian noise.
To test generality beyond this setting, we evaluate the nonlinear deblurring problem (using the neural network degradation forward model from~\cite{chung2023dps}).
Table~\ref{table:comparison_k800} compares SGPnP-RED with DPS.
SGPnP-RED improves PSNR ($26.29$ vs.\ $25.40$ dB) and SSIM ($0.772$ vs.\ $0.742$), while DPS attains a lower LPIPS ($0.220$ vs.\ $0.269$).
This shows SGPnP extends to a nonlinear forward model, though we do not claim uniform improvement across metrics.

\subsection{Denoising at Interpolated Time Steps}
\label{sub:interp_fractional}

Section~\ref{sec:score_adaptation} conditions the pretrained network on a fractional time step $t^*$ obtained by inverting the interpolated noise schedule.
To test whether such fractional steps stay within the trained regime, we corrupt $100$ held-out test images with Gaussian noise at the level $\sigma$ corresponding to $t^*$ and denoise them with the network conditioned on $t^*$, sweeping $t^*$ in quarter steps at four points of the schedule.
Table~\ref{tab:interp_fractional} shows that PSNR decreases monotonically with $\sigma$, with no discontinuity at non-integer $t^*$.

\subsection{Comparison with Flow-Matching-Based PnP}
\label{sub:comparison_pnpflow}

For a matched-dataset comparison with PnP-Flow~\cite{martin2025pnpflow}, we train a flow-matching model on FFHQ and tune PnP-Flow by a per-problem grid search over its hyperparameters.
Table~\ref{table:comparison_pnpflow} shows that SGPnP-PGM outperforms PnP-Flow in PSNR (by $1.2$--$3.9$~dB) and SSIM on all three problems, and in LPIPS on inpainting and deblurring; on super-resolution, PnP-Flow attains a marginally lower LPIPS ($0.193$ vs.\ $0.194$).
Methodologically, both methods re-noise the iterate before denoising: PnP-Flow re-projects it onto the learned flow-matching path by interpolation with Gaussian noise, which ties the noise level to the time at which its denoiser is evaluated, whereas SGPnP adds noise at level $\sigma^{\text{inject}}$ and applies a pretrained score-based diffusion model at a conditioning level $\sigma^{\text{cond}}$ that can be set independently.

\section{Prior Work and Distinction of Our Approach}
\label{app:difference_with_others}

Among recent PnP approaches for inverse problems, stochastic denoising regularization (SNORE)~\cite{renaud2024snore} and our previous work~\cite{park2024scorepnp} are the most closely related. 
While \cite{park2024scorepnp} incorporated score-based diffusion priors into PnP strictly deterministically, SNORE~\cite{renaud2024snore} introduced noise injection for classical denoisers operating exclusively in low-noise regimes within RED. 
In contrast, SGPnP leverages score-based diffusion priors covering a wide noise-level range, enabling it to achieve superior reconstruction performance and solve severely ill-posed inverse problems that previous optimization-based PnP methods cannot address. 
Furthermore, SGPnP extends stochastic re-noising across general PnP frameworks (including ADMM, HQS, PGM, and RED) and provides the first strict saddle-point avoidance guarantee in the PnP literature.
These distinctions are summarized in Table~\ref{tab:comparison_wth_snore}.

\begin{figure}[t]
\begin{center}
\includegraphics[width=0.49\textwidth]{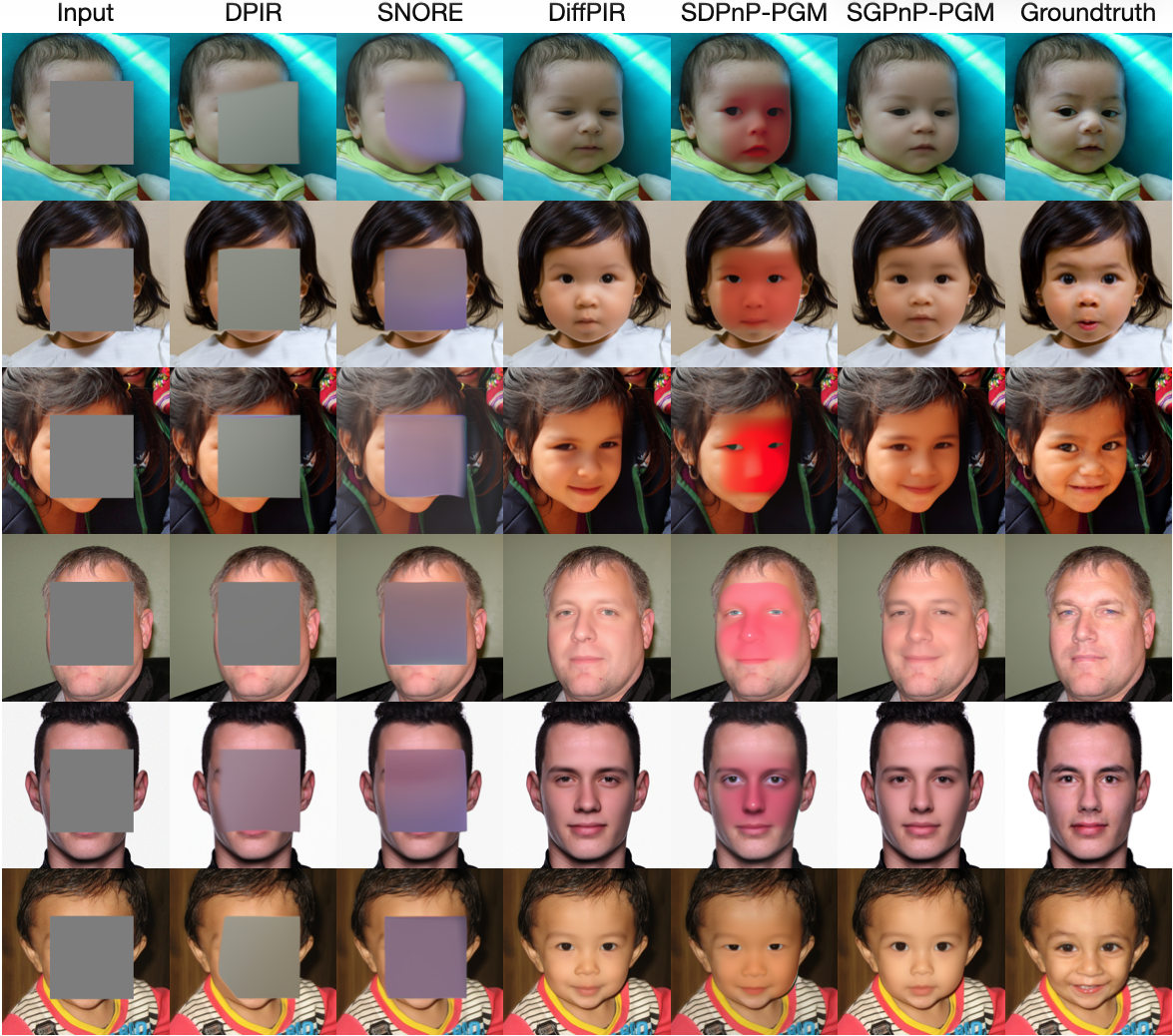}
\end{center}
\caption{
Additional box inpainting results on more measurements.
For this challenging task, DPIR and SNORE often produce incomplete reconstructions. SDPnP-PGM improves the result but still remains incomplete in some cases, whereas SGPnP-PGM produces more plausible image completions.
}
\vspace{-.3cm}
\label{fig:ablation_multi_inpainting}
\end{figure}

\begin{table}[htbp]
\centering
\caption{Comparison of SNORE~\cite{renaud2024snore} and SGPnP.
SGPnP combines stochastic re-noising with score-based diffusion priors and provides guarantees for the resulting stochastic dynamics.}
\label{tab:comparison_wth_snore}
\setlength{\tabcolsep}{5pt}
\renewcommand{\arraystretch}{1.15}
\resizebox{\columnwidth}{!}{
\begin{tabular}{p{2.6cm} p{3.6cm} p{3.9cm}}
\toprule
 & \textbf{SNORE} & \textbf{SGPnP (Ours)} \\
\midrule

\textbf{Denoiser prior}
& Classical denoiser, trained for low-noise restoration.
& Score-based diffusion prior (SBDM), trained across a wide range of noise levels. \\
\midrule

\textbf{Noise injection}
& Injects Gaussian noise and denoises at the \emph{same} noise level.
& Injects Gaussian noise while \emph{decoupling} injected and conditioning noise levels
to account for corruption introduced by data-consistency updates. \\
\midrule

\textbf{Theoretical role of noise}
& Enables convergence toward critical points.
& Promotes escape from strict saddle points while enabling convergence toward critical points. \\
\midrule

\textbf{Severely ill-posed tasks}
& Typically yields modest improvements over deterministic PnP.
& Enables more robust generative reconstruction via expressive SBDM priors
and decoupled noise control. \\
\bottomrule
\end{tabular}
}
\end{table}

\section{Acknowledgement}

This work was supported in part by the National Science Foundation under CAREER award 2625643 and award 2622128, and in part by the U.S. Department of Energy, Office of Science, Office of Advanced Scientific Computing Research, under Triad National Security, LLC (`Triad') contract 89233218CNA000001 [FWP: LANLE2A2] and under Award DE-SC0025589.

\bibliographystyle{IEEEbib}
\bibliography{refs}

@inproceedings{Venkatakrishnan.etal2013,
  title={Plug-and-Play Priors for Model Based Reconstruction},
  author={Venkatakrishnan, Singanallur V. and Bouman, Charles A. and Wohlberg, Brendt},
  booktitle={IEEE Global Conference on Signal and Information Processing},
  pages={945--948},
  year={2013},
}

@Article{Sreehari.etal2016,
  Title                    = {Plug-and-Play Priors for Bright Field Electron Tomography and Sparse Interpolation},
  Author                   = {Sreehari, S. and Venkatakrishnan, S. V. and Wohlberg, B. and Buzzard, G. T. and Drummy, L. F. and Simmons, J. P. and Bouman, C. A.},
  Journal                  = {IEEE Trans. Comput. Imag.},
  Year                     = {2016},

  Month                    = {December},
  Number                   = {4},
  Pages                    = {408--423},
  Volume                   = {2}
}

@article{Parikh.Boyd2014,
  title={Proximal Algorithms},
  author={Parikh, Neal and Boyd, Stephen},
  journal={Foundations and Trends® in Optimization},
  volume={1},
  number={3},
  pages={127--239},
  year={2014},
  publisher={Now Publishers}
}

@article{Chan.etal2016,
  title={Plug-and-Play {ADMM} for Image Restoration: Fixed-Point Convergence and Applications},
  author={Chan, Stanley H. and Wang, Xiran and Elgendy, Omar A.},
  journal={IEEE Transactions on Computational Imaging},
  volume={3},
  number={1},
  pages={84--98},
  year={2016},
  publisher={IEEE}
}

@article{sun2019onlinePnP,
  title={An online plug-and-play algorithm for regularized image reconstruction},
  author={Sun, Yu and Wohlberg, Brendt and Kamilov, Ulugbek S.},
  journal={IEEE Transactions on Computational Imaging},
  volume={5},
  number={3},
  pages={395--408},
  year={2019},
  publisher={IEEE}
}

@inproceedings{
  song2021sde,
  title={Score-Based Generative Modeling through Stochastic Differential Equations},
  author={Yang Song and Jascha Sohl-Dickstein and Diederik P Kingma and Abhishek Kumar and Stefano Ermon and Ben Poole},
  booktitle={International Conference on Learning Representations},
  year={2021},
}

@article{vincent2011connection,
  title={A connection between score matching and denoising autoencoders},
  author={Vincent, Pascal},
  journal={Neural computation},
  volume={23},
  number={7},
  pages={1661--1674},
  year={2011},
  publisher={MIT Press}
}

@inproceedings{ho_NEURIPS2020_ddpm,
  title={Denoising Diffusion Probabilistic Models},
  author={Ho, Jonathan and Jain, Ajay and Abbeel, Pieter},
  booktitle={Advances in Neural Information Processing Systems},
  volume={33},
  pages={6840--6851},
  year={2020}
}

@inproceedings{chung2023dps,
title={Diffusion Posterior Sampling for General Noisy Inverse Problems},
author={Hyungjin Chung and Jeongsol Kim and Michael Thompson McCann and Marc Louis Klasky and Jong Chul Ye},
booktitle={Proc. ICLR},
year={2023}
}

@inproceedings{mike2023score,
  title={Score-Based Diffusion Models for {Bayesian} Image Reconstruction},
  author={M. T. McCann and H. Chung and J. C. Ye and M. L. Klasky},
  booktitle={2023 IEEE International Conference on Image Processing (ICIP)},
  year={2023},
  address={Kuala Lumpur, Malaysia},
  pages={111--115},
  doi={10.1109/ICIP49359.2023.10222481}
}

@inproceedings{kawar2022ddrm,
  title={Denoising diffusion restoration models},
  author={Kawar, Bahjat and Elad, Michael and Ermon, Stefano and Song, Jiaming},
  booktitle={Advances in Neural Information Processing Systems},
  volume={35},
  pages={23593--23606},
  year={2022}
}

@article{efron2011tweedie,
  title={Tweedie’s formula and selection bias},
  author={Efron, Bradley},
  journal={Journal of the American Statistical Association},
  volume={106},
  number={496},
  pages={1602--1614},
  year={2011},
  publisher={Taylor \& Francis}
}

@inproceedings{karras2019ffhq,
  title={A style-based generator architecture for generative adversarial networks},
  author={Karras, Tero and Laine, Samuli and Aila, Timo},
  booktitle={Proceedings of the IEEE/CVF Conference on Computer Vision and Pattern Recognition},
  pages={4401--4410},
  year={2019}
}

@article{romano2017RED,
  title={The little engine that could: regularization by denoising {(RED)}},
  author={Romano, Yaniv and Elad, Michael and Milanfar, Peyman},
  journal={SIAM Journal on Imaging Sciences},
  volume={10},
  number={4},
  pages={1804--1844},
  year={2017},
  publisher={SIAM}
}

@article{zhang2021dpir,
  title={Plug-and-play image restoration with deep denoiser prior},
  author={Zhang, Kai and Li, Yawei and Zuo, Wangmeng and Zhang, Lei and Van Gool, Luc and Timofte, Radu},
  journal={IEEE Transactions on Pattern Analysis and Machine Intelligence},
  volume={44},
  number={10},
  pages={6360--6376},
  year={2021},
  publisher={IEEE}
}

@inproceedings{zhu2023DiffPIR,
  title={Denoising diffusion models for plug-and-play image restoration},
  author={Zhu, Yuanzhi and Zhang, Kai and Liang, Jingyun and Cao, Jiezhang and Wen, Bihan and Timofte, Radu and Van Gool, Luc},
  booktitle={Proceedings of the IEEE/CVF Conference on Computer Vision and Pattern Recognition},
  pages={1219--1229},
  year={2023}
}

@article{sun2024provablePMC,
  title={Provable probabilistic imaging using score-based generative priors},
  author={Sun, Yu and Wu, Zihui and Chen, Yifan and Feng, Berthy T. and Bouman, Katherine L.},
  journal={IEEE Transactions on Computational Imaging},
  year={2024}
}

@article{dabov2007bm3d,
  title={Image denoising by sparse 3-D transform-domain collaborative filtering},
  author={Dabov, Kostadin and Foi, Alessandro and Katkovnik, Vladimir and Egiazarian, Karen},
  journal={IEEE Transactions on Image Processing},
  volume={16},
  number={8},
  pages={2080--2095},
  year={2007},
  publisher={IEEE}
}

@article{laumont2022pnpula,
  title={Bayesian imaging using plug \& play priors: when {L}angevin meets {T}weedie},
  author={Laumont, R{\'e}mi and De Bortoli, Valentin and Almansa, Andr{\'e}s and Delon, Julie and Durmus, Alain and Pereyra, Marcelo},
  journal={SIAM Journal on Imaging Sciences},
  volume={15},
  number={2},
  pages={701--737},
  year={2022},
  publisher={SIAM}
}

@inproceedings{wu2024principledPnP,
  title={Principled Probabilistic Imaging using Diffusion Models as Plug-and-Play Priors},
  author={Wu, Zihui and Sun, Yu and Chen, Yifan and Zhang, Bingliang and Yue, Yisong and Bouman, Katherine L.},
  booktitle={Proceedings of the 38th Conference on Neural Information Processing Systems (NeurIPS)},
  year={2024}
}

@inproceedings{bouman2023generativePnP,
  title={Generative plug and play: Posterior sampling for inverse problems},
  author={Bouman, Charles A. and Buzzard, Gregery T.},
  booktitle={2023 59th Annual Allerton Conference on Communication, Control, and Computing (Allerton)},
  pages={1--7},
  year={2023},
  organization={IEEE}
}

@article{xu2024provably_dps_pnp,
  title={Provably robust score-based diffusion posterior sampling for plug-and-play image reconstruction},
  author={Xu, Xingyu and Chi, Yuejie},
  journal={arXiv:2403.17042},
  year={2024}
}

@misc{tachella2023deepinverse,
  author={Tachella, Julian and Chen, Dongdong and Hurault, Samuel and Terris, Matthieu and Wang, Andrew},
  title={{DeepInverse}: A deep learning framework for inverse problems in imaging},
  year={2023},
  version={latest},
  doi={10.5281/zenodo.7982256},
  note={Date released: 2023-06-30}
}

@article{park2024randomwalks,
  title = {Random Walks With {T}weedie: A Unified View of Score-Based Diffusion Models [In the Spotlight]},
  author = {Park, Chicago Y. and McCann, Michael T. and Garcia-Cardona, Cristina and Wohlberg, Brendt and Kamilov, Ulugbek S.},
  journal = {IEEE Signal Processing Magazine},
  volume = {42},
  number = {3},
  pages = {40--51},
  year = {2025},
  publisher = {IEEE},
  doi = {10.1109/MSP.2025.3590608},
  url = {https://doi.org/10.1109/MSP.2025.3590608}
}

@inproceedings{park2024measurementdiffusion,
  title={Measurement Score-Based Diffusion Model},
  author = {Park, Chicago Y. and Shoushtari, Shirin and An, Hongyu and Kamilov, Ulugbek S.},
  booktitle={International Conference on Learning Representations},
  year={2026}
}

@inproceedings{park2024scorepnp,
  title={Plug-and-Play Priors as a Score-Based Method},
      author={Park, Chicago Y.
        and Yuyang Hu
        and McCann, Michael T.
        and Garcia-Cardona, Cristina
        and Wohlberg, Brendt
        and Kamilov, Ulugbek S.},
  booktitle={IEEE International Conference on Image Processing},
  year={2025},
  address={Anchorage, Alaska},
}

@inproceedings{zhang2018lpips,
  title={The Unreasonable Effectiveness of Deep Features as a Perceptual Metric},
  author={Zhang, Richard and Isola, Phillip and Efros, Alexei A. and Shechtman, Eli and Wang, Oliver},
  booktitle={Proceedings of the IEEE/CVF Conference on Computer Vision and Pattern Recognition},
  pages={586--595},
  year={2018}
}

@article{coeurdoux2024pnpsplitgibbssampler,
  title={Plug-and-Play Split {G}ibbs Sampler: Embedding Deep Generative Priors in Bayesian Inference},
  author={Coeurdoux, Florentin and Dobigeon, Nicolas and Chainais, Pierre},
  journal={IEEE Transactions on Image Processing},
  year={2024},
  publisher={IEEE}
}

@article{faye2024REDbayesian,
  title={Regularization by denoising: Bayesian model and {L}angevin-within-split {G}ibbs sampling},
  author={Faye, Elhadji C. and Fall, Mame Diarra and Dobigeon, Nicolas},
  journal={arXiv:2402.12292},
  year={2024}
}

@inproceedings{renaud2024snore,
  title={Plug-and-Play Image Restoration with Stochastic Denoising Regularization},
  author={Renaud, Marien and Prost, Jean and Leclaire, Arthur and Papadakis, Nicolas},
  booktitle={Proceedings of the 41st International Conference on Machine Learning},
  year={2024}
}

@article{knoll2020fastmri2,
  title={{fastMRI}: A Publicly Available Raw k-Space and {DICOM} Dataset of Knee Images for Accelerated {MR} Image Reconstruction Using Machine Learning},
  author={Knoll, Florian and Zbontar, Jure and Sriram, Anuroop and Muckley, Matthew J. and Bruno, Mary and Defazio, Aaron and Parente, Marc and Geras, Krzysztof J. and Katsnelson, Joe and Chandarana, Hersh and Zhang, Zizhao and Drozdzal, Michal and Romero, Adriana and Rabbat, Michael and Vincent, Pascal and Pinkerton, James and Wang, Duo and Yakubova, Nafissa and Owens, Erich and Zitnick, C. Lawrence and Recht, Michael P. and Sodickson, Daniel K. and Lui, Yvonne W.},
  journal={Radiology: Artificial Intelligence},
  volume={2},
  number={1},
  pages={e190007},
  year={2020},
  publisher={Radiological Society of North America},
  doi={10.1148/ryai.2020190007},
}

@article{zbontar2018fastmri1,
  title={{fastMRI}: An Open Dataset and Benchmarks for Accelerated {MRI}},
  author={Zbontar, Jure and Knoll, Florian and Sriram, Anuroop and Murrell, Tullie and Huang, Zhengnan and Muckley, Matthew J. and Defazio, Aaron and Stern, Ruben and Johnson, Patricia and Bruno, Mary and Parente, Marc and Geras, Krzysztof J. and Katsnelson, Joe and Chandarana, Hersh and Zhang, Zizhao and Drozdzal, Michal and Romero, Adriana and Rabbat, Michael and Vincent, Pascal and Yakubova, Nafissa and Pinkerton, James and Wang, Duo and Owens, Erich and Zitnick, C. Lawrence and Recht, Michael P. and Sodickson, Daniel K. and Lui, Yvonne W.},
  journal={arXiv:1811.08839},
  year={2018}
}

@article{pemantle1990nonconvergence,
  title={Nonconvergence to Unstable Points in {Urn} Models and Stochastic Approximations},
  author={Pemantle, Robin},
  journal={The Annals of Probability},
  volume={18},
  number={2},
  pages={698--712},
  year={1990},
  publisher={Institute of Mathematical Statistics}
}

@article{daras2024survey,
  title={A Survey on Diffusion Models for Inverse Problems},
  author={Daras, Giannis and Chung, Hyungjin and Lai, Chieh-Hsin and Mitsufuji, Yuki and Ye, Jong Chul and Milanfar, Peyman and Dimakis, Alexandros G. and Delbracio, Mauricio},
  journal={arXiv:2410.00083},
  year={2024}
}

@article{tang2020fast,
  title={A Fast Stochastic Plug-and-Play {ADMM} for Imaging Inverse Problems},
  author={Tang, Junqi and Davies, Mike},
  journal={arXiv:2006.11630},
  year={2020}
}

@article{sun2018plugin,
  title={Plug-in Stochastic Gradient Method},
  author={Sun, Yu and Wohlberg, Brendt and Kamilov, Ulugbek S.},
  journal={arXiv:1811.03659},
  year={2018}
}

@article{liu2021sgdnet,
  title={{SGD-Net}: Efficient Model-Based Deep Learning with Theoretical Guarantees},
  author={Liu, Jiaming and Sun, Yu and Gan, Weijie and Xu, Xiaojian and Wohlberg, Brendt and Kamilov, Ulugbek S.},
  journal={IEEE Transactions on Computational Imaging},
  volume={7},
  pages={598--610},
  year={2021},
  publisher={IEEE}
}

@article{sun2020asyncred,
  title={Async-{RED}: A Provably Convergent Asynchronous Block Parallel Stochastic Method Using Deep Denoising Priors},
  author={Sun, Yu and Liu, Jiaming and Sun, Yiran and Wohlberg, Brendt and Kamilov, Ulugbek S.},
  journal={arXiv:2010.01446},
  year={2020}
}

@inproceedings{wu2019online,
  title={Online Regularization by Denoising with Applications to Phase Retrieval},
  author={Wu, Zihui and Sun, Yu and Liu, Jiaming and Kamilov, Ulugbek S.},
  booktitle={Proceedings of the IEEE/CVF International Conference on Computer Vision Workshops},
  year={2019}
}

@article{sun2021scalable,
  title={Scalable Plug-and-Play {ADMM} with Convergence Guarantees},
  author={Sun, Yu and Wu, Zihui and Xu, Xiaojian and Wohlberg, Brendt and Kamilov, Ulugbek S.},
  journal={IEEE Transactions on Computational Imaging},
  volume={7},
  pages={849--863},
  year={2021},
  publisher={IEEE}
}

@inproceedings{daneshmand2018escaping,
  title={Escaping Saddles with Stochastic Gradients},
  author={Daneshmand, Hadi and Kohler, Jonas and Lucchi, Aurelien and Hofmann, Thomas},
  booktitle={Proceedings of the International Conference on Machine Learning},
  pages={1155--1164},
  year={2018},
  organization={PMLR}
}

@article{li2023restarted,
  title   = {Restarted Nonconvex Accelerated Gradient Descent: No More Polylogarithmic Factor in the {$O(\epsilon^{-7/4})$} Complexity},
  author  = {Li, Huan and Lin, Zhouchen},
  journal = {Journal of Machine Learning Research},
  volume  = {24},
  number  = {157},
  pages   = {1--37},
  year    = {2023}
}

@inproceedings{carmon2017convex,
  title     = {{``Convex until proven guilty''}: Dimension-free acceleration of gradient descent on non-convex functions},
  author    = {Carmon, Yair and Duchi, John C. and Hinder, Oliver and Sidford, Aaron},
  booktitle = {Proceedings of the 34th International Conference on Machine Learning},
  pages     = {654--663},
  year      = {2017},
  publisher = {PMLR}
}

@inproceedings{agarwal2017finding,
  title     = {Finding Approximate Local Minima Faster than Gradient Descent},
  author    = {Agarwal, Naman and Allen-Zhu, Zeyuan and Bullins, Brian and Hazan, Elad and Ma, Tengyu},
  booktitle = {Proceedings of the 49th Annual ACM SIGACT Symposium on Theory of Computing},
  pages     = {1195--1199},
  year      = {2017}
}

@article{marumo2024parameter,
  title   = {Parameter-free Accelerated Gradient Descent for Nonconvex Minimization},
  author  = {Marumo, Naoki and Takeda, Akiko},
  journal = {SIAM Journal on Optimization},
  volume  = {34},
  number  = {2},
  pages   = {2093--2120},
  year    = {2024}
}

@article{gulle2025consistency,
  title     = {Consistency Models as Plug-and-Play Priors for Inverse Problems},
  author    = {G{\"u}lle, Merve and Yun, Junno and Al{\c{c}}alar, Ya{\c{s}}ar Utku and Ak{\c{c}}akaya, Mehmet},
  journal   = {arXiv:2509.22736},
  year      = {2025}
}

@incollection{combettes2011proximal,
  title     = {Proximal Splitting Methods in Signal Processing},
  author    = {Combettes, Patrick L. and Pesquet, Jean-Christophe},
  booktitle = {Fixed-Point Algorithms for Inverse Problems in Science and Engineering},
  pages     = {185--212},
  year      = {2011},
  publisher = {Springer New York},
  address   = {New York, NY}
}

@article{eslahi2022improved,
  author  = {Eslahi, Nasser and Foi, Alessandro},
  title   = {Improved Sparse Signal Recovery via Adaptive Correlated Noise Model},
  journal = {IEEE Transactions on Computational Imaging},
  volume  = {8},
  pages   = {945--960},
  year    = {2022}
}

@article{chand2024multiscale,
  author  = {Chand, Jyothi Rikhab and Jacob, Mathews},
  title   = {Multi-scale Energy ({MUSE}) Framework for Inverse Problems in Imaging},
  journal = {IEEE Transactions on Computational Imaging},
  volume  = {10},
  pages   = {1250--1265},
  year    = {2024},
  doi     = {10.1109/TCI.2024.3496030}
}

@article{bendel2025solving,
  author  = {Bendel, Matthew C. and Shastri, Saurav K. and Ahmad, Rizwan and Schniter, Philip},
  title   = {Solving Inverse Problems using Diffusion with Iterative Colored Renoising},
  journal = {Transactions on Machine Learning Research},
  volume  = {2025},
  pages   = {4646},
  year    = {2025},
  url     = {https://openreview.net/forum?id=uN6A3b8Abe}
}

@inproceedings{martin2025pnpflow,
  author    = {Martin, S{\'e}gol{\`e}ne and Gagneux, Anne and Hagemann, Paul and Steidl, Gabriele},
  title     = {PnP-Flow: Plug-and-Play Image Restoration with Flow Matching},
  booktitle = {International Conference on Learning Representations (ICLR)},
  pages     = {45466--45492},
  year      = {2025}
}

\end{document}